\documentclass[nofootinbib,onecolumn]{revtex4-1}
\usepackage[pdftex]{graphicx}
\setkeys{Gin}{clip=true,draft=false}
\DeclareGraphicsExtensions{.pdf}
\usepackage[bottom]{footmisc}
\usepackage{amsmath}
\usepackage{amsfonts}
\usepackage{amsthm}
\usepackage{mathrsfs}
\usepackage{algorithm,algpseudocode}
\usepackage{caption}
\usepackage{subcaption}
\usepackage{color}
\usepackage{braket}
\usepackage{yfonts}
\usepackage[inline]{enumitem}
\usepackage{centernot}
\newtheorem{theorem}{Theorem}[section]

\newtheorem{proposition}[theorem]{Proposition} 

\newtheorem{definition}{Definition}[section]

\newenvironment{remark}{\par\noindent{\bf Remark.\ } \it}

\newcommand{\mbf}{\mathbf}
\newcommand{\mbb}{\mathbb}
\newcommand{\cond}{\,\,|\,\,}
\newcommand{\bcond}{\,\,\big|\,\,}
\newcommand{\argmax}{\text{argmax}}

\newcommand{\E}{\mbf{E}}
\newcommand{\Var}{\mathrm{Var}}
\newcommand{\erf}{\mathrm{erf}}
\newcommand{\erfc}{\mathrm{erfc}}

\begin{document}

\title{Generating Weighted MAX-2-SAT Instances with Frustrated Loops: an RBM Case Study}

\author{Yan Ru Pei}
\email{email: yrpei@ucsd.edu}
\affiliation{Department of Physics, University of California, San Diego, La Jolla, CA 92093}

\author{Haik Manukian}
\email{email: hmanukia@ucsd.edu}
\affiliation{Department of Physics, University of California, San Diego, La Jolla, CA 92093}

\author{M. Di Ventra}
\email{email: diventra@physics.ucsd.edu}
\affiliation{Department of Physics, University of California, San Diego, La Jolla, CA 92093}

\begin{abstract}
Many optimization problems can be cast into the maximum satisfiability (MAX-SAT) form, and many solvers have been developed for tackling such problems. To evaluate a MAX-SAT solver, it is convenient to generate hard MAX-SAT instances with known solutions. Here, we propose a method of generating weighted MAX-2-SAT instances inspired by the frustrated-loop algorithm used by the quantum annealing community. We extend the algorithm for instances of general bipartite couplings, with the associated optimization problem being the minimization of the restricted Boltzmann machine (RBM) energy over the nodal values, which is useful for effectively pre-training the RBM. The hardness of the generated instances can be tuned through a central parameter known as the frustration index. Two versions of the algorithm are presented: the random- and structured-loop algorithms. For the random-loop algorithm, we provide a thorough theoretical and empirical analysis on its mathematical properties from the perspective of frustration, and observe empirically a double phase transition behavior in the hardness scaling behavior driven by the frustration index. For the structured-loop algorithm, we show that it offers an improvement in hardness over the random-loop algorithm in the regime of high loop density, with the variation of hardness tunable through the concentration of frustrated weights.
\end{abstract}

\maketitle

\section{Introduction}

A boolean satisfiability problem is the problem of finding a truth assignment of the boolean variables in a boolean formula such that the formula is satisfied. A boolean formula is commonly written in conjunctive normal form (CNF) consisting of clauses. A clause is formed by the logical disjunction (OR) of literals (boolean variables or their negations), and the boolean formula is expressed as the logical conjunction (AND) of these clauses. In other words, a clause is satisfied if at least one of its literals evaluates to true, and all clauses need to be satisfied in order to solve the boolean satisfiability problem.

A maximum-satisfiability (MAX-SAT) problem is an optimization problem where the objective is to find the truth assignment of boolean variables such that the number of satisfied clauses is maximized (without requiring the satisfaction of all clauses). A Max-2-SAT problem is a Max-SAT problem with at most 2 literals (variables or their negation) per clause \cite{max2sat}. Many optimization problems can be reduced to this particular problem \cite{2sat_re,2sat_ising,max_re2}, making it a valuable testing ground for various algorithms/solvers. In this paper, we focus on the weighted MAX-2-SAT problem, which is a more general version of the MAX-2-SAT problem where each clause is assigned some non-negative weight \cite{boo_com}, and the objective is to find the truth assignments of the literals that maximizes the combined weight of the satisfied clauses.

The optimization version of the (weighted) MAX-2-SAT problem is known to be NP-hard \cite{2satnp}, and it is difficult to check whether a solver has found the optimal solution for a given instance. This makes the evaluation of the performance of a MAX-SAT solver rather impractical. An algorithm that generates Max-2-SAT instances of tunable hardness such that the solution is known in advance (planted solution) would therefore be very beneficial. Such algorithms exist, but they generally suffer from one or more of the following drawbacks: they are unable to generate sufficiently hard instances \cite{draw_1}, the planted solution is not necessarily the optimal solution \cite{draw_2}, they are limited in the structures of the instances that they can generate \cite{loop}, or they require a considerable amount of computational time to generate the instances \cite{draw_3}.

In this paper, we introduce an algorithm that is capable of generating weighted Max-2-SAT instances of tunable hardness in both the low and high clause density regimes which does not suffer from the above limitations. For ease of theoretical analysis and empirical studies, we can reduce a given weighted MAX-2-SAT problem into the following problem (see Section \ref{2spin})
\begin{equation}
\label{RBM-Eq}
\begin{split}
&\text{Find a configuration of } \mathbf{v}\in \{-1,1\}^n \text{ and } \mathbf{h}\in \{-1,1\}^m \\
&\text{such that the following energy function is minimized:} \\
&E ( \mathbf{v},\mathbf{h}) = -\big( \sum_{ij} W_{ij} v_ih_j + \sum_i a_iv_i + \sum_j b_jh_j \big),
\end{split}
\end{equation}
where $\mathbf{a}\in \mathbb{R}^n$, $\mathbf{b}\in \mathbb{R}^m$, and $\mathbf{W}\in \mathbb{R}^{n\times m}$. We refer to the above problem for some given $\{ \mbf{a}, \mbf{b}, \mbf{W} \}$ as an {\it RBM instance}. Note that the reduction from a MAX-2-SAT instance to an RBM instance is always possible (see Section \ref{2spin}), meaning that it is sufficient for us to focus on generating RBM instances. However, this reduction involves doubling the number of boolean variables and introducing clauses of large weights (see Section \ref{bi_conv}), thus introducing unnecessary computational and memory burden in a practical implementation of the algorithm. Therefore, it is still useful to have an algorithm capable of generating instances directly corresponding to an underlying graph structure. In this paper, we focus on generating instances directly on bipartite graphs, though our algorithm can be easily extended to any underlying graph structure (see Section \ref{ext}).

The quantity $E(\mathbf{v},\mathbf{h})$ as appeared in Eq.~(\ref{RBM-Eq}) is also called the {\it RBM energy} in the field of machine learning, because it is the energy function for an RBM (restricted Boltzmann machine) \cite{rbm}. The nodal configuration with the minimum RBM energy (or {\it RBM ground state}) corresponds to the mode of the RBM joint PMF (probability mass function) \cite{rbm, prob}. Finding the mode of the joint PMF allows for a much more efficient sampling of the PMF \cite{multi_mode, mode_mcmc, mode_jump}, which helps improve the RBM pre-training in both decreasing the number of iterations to convergence and minimizing the KL-divergence \cite{kl,dart,worm}. A brief extension of this discussion will be presented in Section \ref{app}, and a more thorough treatment of this topic will be given in our related work \cite{mode_train}.

To generate RBM instances with a known global optimum and tunable hardness, we take inspiration from the {\it frustrated-loop algorithm} used by the quantum annealing community to benchmark the performance of quantum annealers \cite{loop}. However, due to the high connectivity and non-local coupling nature of the RBM instance, the frustrated-loop algorithm in its original form is unable to provide sufficient hardness for instances at high density; this is because the positive and negative edges will intersect frequently, resulting in instances with a lower frustration index than the one desired (see Section \ref{int}). In this work, we modify the loop algorithm in a way such that the frustration index can be controlled and tuned directly (see Section \ref{tune}), which allows us to generate instances of desired hardness.

This paper is organized as follows. In Section \ref{spin-glass-RBM}, we show explicitly the reduction of a general weighted MAX-2-SAT problem into an RBM instance with Ising-type coupling \cite{Mezard}. In Section \ref{frustration}, we introduce the {\it frustration index} \cite{frus} in the context of a {\it gauged RBM}, and discuss its connection to the population of local minima. In Section \ref{frus_loop_algo}, we introduce the {\it random} frustrated loop algorithm and propose a general method for direct control of the frustration index. We also investigate some of its interesting mathematical properties, and discuss the limitations of the algorithm in its original form. In Section \ref{modi_algo}, we make improvements on the algorithm by giving the loops certain geometrical structures, and we term the new algorithm {\it structured} loop algorithm. We show analytically that the new algorithm has the ability to generate hard instances in the regime of high loop density. In Section \ref{test}, we study empirically how the hardness of the generated instances scale with RBM size, frustration index, and loop density. We observe the {\it hardness peaks} \cite{easy-hard-easy} with respect to the loop density for systems of different sizes, and present a double phase transition \cite{phase} driven by the frustration index in the hardness scaling behavior with respect to the system size. We also provide empirical justification of the hardness improvement of the structured loop algorithm over the random counterpart. In Section \ref{app}, we offer a heuristic discussion on some practical applications of this work to the task of pre-training an RBM.

If the interest of reader is only the generation of frustrated RBM instances, we direct the reader to Section \ref{secfrusalgo} and \ref{sec_slog} respectively for the random and structured loop algorithms, where the pseudocode for the generation method is given. The MATLAB implementation of both algorithms are available in the Github repository \path{PeaBrane/Ising-Simulation}.

\section{From Weighted MAX-2-SAT to Bipartite Spin Glass}\label{spin-glass-RBM}
\label{2spin}

The goal of this section is to show that any given weighted MAX-2-SAT instance can be reduced to an RBM instance with Ising-type couplings (meaning that the nodal values are $\{-1,1\}$). Similarly, any given RBM instance can also be converted back into a MAX-2-SAT instance. This means that we can effectively test the performance of a weighted MAX-2-SAT solver on the corresponding RBM instance.

The conversion from an RBM instance to a weighted MAX-2-SAT problem is relatively straightforward and will be described in Section \ref{inv_conv}. The reduction from a weighted MAX-2-SAT instance to an RBM instance is less obvious and is done in three stages. First, we reduce a general weighted MAX-2-SAT instance into a QUBO (quadratic binary optimization) problem \cite{qubo}. Then, we reduce the QUBO problem into a larger QUBO problem of bipartite form. Finally, we convert the binary node values of $\{0,1\}$ into $\{-1,1\}$, and introduce two extra spins to incorporate the biases. 

\subsection{From Weighted MAX-2-SAT to QUBO}

A QUBO instance is the problem of maximizing the following quadratic polynomial
\begin{equation}
\label{qubo}
\sum_{i=1}^nB_i x_i + \sum_{i=1}^n\sum_{j=i+1}^nQ_{ij} x_i x_j
\end{equation}
over the binary variables $\mathbf{x}\in \{0,1\}^n$, with the coefficients $B_i$ and $Q_{ij}$ being real. The reduction from a weighted MAX-2-SAT problem to a QUBO problem is rather straightforward, and it simply involves converting each clause into the equivalent QUBO form
\begin{equation*}
\begin{split}
(x_i \lor x_j) &\rightarrow x_i + x_j - x_ix_j, \\
(\neg x_i \lor x_j) &\rightarrow 1 - x_i + x_ix_j, \\
(x_i \lor \neg x_j) &\rightarrow 1 - x_j + x_ix_j, \\
(\neg x_i \lor \neg x_j) &\rightarrow 1 - x_ix_j. \\
\end{split}
\end{equation*}
The correctness of the conversion can be easily verified by treating a true assignment as $1$ and a false assignment as $0$, and plugging the corresponding binary values into the right-hand side (RHS) of the above expressions. We then sum the QUBO terms corresponding to all the clauses weighted respectively, and the resulting expression is in the QUBO form if we ignore the constant offset. It is not hard to see that maximizing the summed weights of the satisfied clauses is equivalent to maximizing the quadratic polynomial in Eq.~(\ref{qubo}).

\subsection{Bipartite Conversion}
\label{bi_conv}

A bipartite QUBO instance involves binary variables in two disjoint sets, which we can denote as $\mbf{v} \in \{0,1\}^n$ and $\mbf{h} \in \{0,1\}^m$, and every quadratic term in the polynomial is composed of a variable from each set
\begin{equation}
\label{bi_qubo}
\sum_{i=1}^n a_iv_i + \sum_{j=1}^m b_jh_j + \sum_{i=1}^n\sum_{j=1}^m W_{ij}v_ih_j,
\end{equation}
where $\mbf{v}\in\{0,1\}^n$ and $\mbf{h}\in\{0,1\}^m$. The underlying bipartite graph is then $K_{n,m}$ for this bipartite QUBO problem.

For any given general QUBO instance (Eq.~(\ref{qubo})) with $n$ variables, we can always convert it into a $K_{n,n}$ bipartite QUBO instance, such that the optimal truth assignment of one set of variables, say $\mbf{v}$, corresponds to the optimal truth assignment of the original QUBO instance. We can let the bipartite QUBO instance be
\begin{equation}
\label{rbi_qubo}
\begin{split}
E(\mbf{v}, \mbf{h}) = & E_0(\mbf{v}, \mbf{h}) + C(\mbf{v},\mbf{h}) \\
= & \sum_{i=1}^n B_iv_i + \sum_{i=1}^n\sum_{j=i+1}^nQ_{ij}v_ih_j + C(\bf{v},\bf{h}),
\end{split}
\end{equation}
with $C(\mbf{v},\mbf{h})$ being some penalty function (in QUBO form) ensuring the invariance of the maximum under the bipartite conversion. Note that since both $E_0$ and $C$ are in QUBO form, the addition of the two constitutes a QUBO instance. 

The purpose of the penalty function $C$ is to ensure that $\mbf{v} = \mbf{h}$ is satisfied at the maximum of $E(\mbf{v},\mbf{h})$. This is done by constructing $C(\mbf{v}, \mbf{h})$ such that whenever an assignment deviates from the condition $\mbf{v} = \mbf{h}$, the function generates a cost large enough to overcome the increase in the value of the polynomial $E_0(\mbf{v}, \mbf{h})$ through the relaxation of the $\mbf{v} = \mbf{h}$ condition. This necessarily implies that the maximum of $E(\mbf{v}, \mbf{h})$ is the same as the maximum of $E(\mbf{v}, \mbf{v})$, with the latter being equivalent to the polynomial of the original QUBO instance (\ref{bi_qubo}) up to a constant offset. Therefore, the maximum of (\ref{bi_qubo}) is the same as the maximum of (\ref{rbi_qubo}), and the correctness of the reduction is guaranteed. An explicit construction of the penalty function $C(\mbf{v}, \mbf{h})$ is given in Appendix \ref{appendix:rbi}. 

\subsection{Conversion to $\{-1,1\}$ Binary Values}
\label{11}

In this work, it is convenient to restrict the binary variables of the optimization problem to be $\{{\bf v'},{\bf h'}\}\in \{-1,1\}^{n+m}$, so that the quadratic terms essentially describe Ising-type couplings between the variables. We shall, from here on, refer to $\mbf{v'}$ as {\it visible spins} and $\mbf{h'}$ as {\it hidden spins}, with the two names descending from the naming conventions of statistical mechanics and machine learning. 

The conversion from the old binary variables, $\{ \mbf{v}, \mbf{h} \} \in \{0,1\}^{n+m}$, to the new binary variables, $\{ \mbf{v'}, \mbf{h'} \} = \{-1,1\}^{n+m}$ values can be simply performed as follows
\begin{equation*}
\mathbf{v'}=-1+2\mathbf{v}, \qquad \mathbf{h'}=-1+2\mathbf{h}.
\end{equation*}
With this conversion, the original QUBO polynomial (see Eq.~(\ref{rbi_qubo})) can be written in terms of the new binary variables as
\begin{equation*}
\begin{split}
E(\mathbf{v'},\mathbf{h'}) &= \sum_{i=1}^n a_i \frac{v_i'+1}{2} + \sum_{j=1}^m b_j \frac{h_j'+1}{2} 
+ \sum_{i=1}^n\sum_{j=1}^m W_{ij}\frac{v_i'+1}{2}\frac{h_j'+1}{2} \\
&= \sum_{i=1}^n\frac{1}{2}(a_i+\sum_{j=1}^m W_{ij})v_i' 
+ \sum_{j=1}^m\frac{1}{2}(b_j + \sum_{i=1}^n W_{ij})h_j' + \sum_{i=1}^n\sum_{j=1}^m\frac{W_{ij}}{4}v'_ih'_j \\
&+ (\sum_{i=1}^n \frac{a_i}{2} + \sum_{j=1}^m \frac{b_j}{2} + \sum_{i=1}^n\sum_{j=1}^m\frac{W_{ij}}{4}).
\end{split}
\end{equation*}
We can define the new linear and quadratic coefficients to be $a_i'=\frac{1}{2}(a_i+\sum_j W_{ij})$, $b_j'=\frac{1}{2}(b_j+\sum_i W_{ij})$, and $W_{ij}'=\frac{W_{ij}}{4}$, respectively. Furthermore, we can choose to ignore the last bracketed term since it is just a constant offset independent of the $\{\mathbf{v'},\mathbf{h'}\}$. The polynomial can then be rewritten as
\begin{equation*}
E'(\mathbf{v'},\mathbf{h'}) = \sum_{i=1}^n a'_i v'_i + \sum_{j=1}^m b'_j h'_j + \sum_{i=1}^n\sum_{j=1}^m W'_{ij}v'_ih'_j,
\end{equation*}
which is in the same form as the original QUBO polynomial. 

From now on, we will discard the prime symbols on the coefficients and the binary variables in the polynomial expression, and we will always assume that the binary variables take values $\{-1,+1\}$. As it is the convention in the field of machine learning, the RBM energy is often equipped with a total negative sign, meaning that maximizing a QUBO polynomial is equivalent to minimizing the corresponding RBM energy
\begin{equation*}
E ( \mathbf{v},\mathbf{h}) = -\big( \sum_{ij} W_{ij} v_ih_j + \sum_i a_iv_i + \sum_j b_jh_j \big).
\end{equation*}
We will also follow this convention in this paper.

\subsection{Biases as Ghost Spins}
\label{ghost}

The linear coefficients in the RBM energy are referred to as {\it biases}. The bias terms can be interpreted as spins interacting with some external field. In some cases, it is convenient to express this interaction as a two-body interaction between a spin and some imaginary fixed spin, or {\it ghost spins} \cite{ghost} with a coupling strength proportional to the external field. 

To be more precise, we can define additional weight elements, $W_{i,m+1} = a_i$ and $W_{n+1,j} = b_j$, and two additional spins, $v_{n+1} = 1$ and $h_{n+1} = 1$, so that the RBM energy can be expressed compactly as
\begin{equation*}
E(\mathbf{v},\mathbf{h}) = -\sum_{i=1}^{n+1}\sum_{j=1}^{m+1} W_{ij}v_ih_j.
\end{equation*}
In this form, the linear terms are absorbed into the quadratic terms, resulting in a fully quadratic expression.

\subsection{Inverse Conversion}
\label{inv_conv}

In addition to converting a weighted Max-2-SAT instance into an RBM instance, the inverse conversion is also possible. Given an RBM instance, one way that the conversion can be performed is by breaking each Ising coupling up into two clauses, with the form of the two clauses dependent on the sign of the coupling. The coupling is broken up as follows
\begin{equation*}
W_{ij}v_ih_j \rightarrow
\begin{cases}
2W_{ij}( v_i \lor \neg h_j) \,\land\, 2W_{ij}( \neg v_i \lor h_j) \quad \text{if} \quad W_{ij} \geq 0, \\
(-2W_{ij})( v_i \lor h_j) \,\land\, (-2W_{ij})( \neg v_i \lor \neg h_j) \quad \text{if} \quad W_{ij} < 0, \\
\end{cases}
\end{equation*}
where $+1$ is interpreted as a true assignment and $-1$ is interpreted as a false assignment; the $\land$ operation can be interpreted as a plus sign. 

It can easily be verified from the RHS of the above conversion that when a bond is satisfied (say $v_i = h_j$ when $W_{ij} > 0$), then both clauses will be satisfied, and if the bond is violated (say $v_i \neq h_j$ when $W_{ij} > 0$), then only one clause will be satisfied. This results in an energy penalty of $2 |W_{ij}|$ (corresponding to the violation of one clause) whenever a bond is violated, and this agrees with the left-hand side (LHS) expression.

If we consider an $n\times m$ RBM with a weight assigned for every pair of visible and hidden spins, then there are clearly $nm$ bonds, which break up into $2nm$ clauses in this conversion scheme. The clause density of the weighted MAX-2-SAT is then
\begin{equation*}
\rho = \frac{2nm}{n+m}.
\end{equation*}
If $m$ scales linearly with $n$, then the clause density is clearly of order $O(n)$, which is extensive with respect to the system size. In this regime of clause density, it is generally difficult to generate planted instances of sufficient hardness if the clauses are formed in a completely random fashion\footnote{In the context of $k$-SAT, this means that we are forming the boolean formula by uniformly sampling clauses agreeing with the planted solution. In the context of RBM instances, this means that the frustrated loops are dropped on the bipartite graph randomly (see Section \ref{limit}).} \cite{poly_3sat}. In Section \ref{Section_struc}, we address this problem by presenting an algorithm that enforces certain weight structures on the RBM, which results in hard instances at high clause density.

\section{Frustration of RBMs}\label{frustration}

In this section we formulate an RBM instance entirely in terms of its corresponding {\it weight matrix}, and express the action of a spin flip as {\it vertex switching} \cite{switch}, which is defined as the negation of the signs of a certain subset of weight elements. We then introduce a measure of the hardness of an RBM instance known as the {\it frustration index} \cite{Mezard, frus} and discuss its relationship to the population of local minima of the RBM instance. For simplicity, we assume from here on that the RBM is {\it unbiased} \cite{rbm} (unless specifically mentioned), meaning that we can set $\mathbf{a}=0$ and $\mathbf{b}=0$. Note that most of the results derived in this paper can be easily generalized to a {\it biased} RBM. 

\subsection{Vertex Switching as Local Gauge}
\label{v_switch}

Given any spin configuration $\mathbf{s}=\{\mathbf{v},\mathbf{h}\}$ of the RBM, we can negate the signs of (or flip) a portion of the spins and arrive at some new configuration $\mathbf{s'}=\{\mathbf{v'},\mathbf{h'}\}$. This is equivalent to negating a certain subset of the weight elements. We formally define the vertex switching operation as follows.

\begin{definition}
Given an $n\times m$ RBM with weight matrix $\mbf{W}$ and two states, $\mbf{s} = ( \mbf{v}, \mbf{h} )$ and $\mbf{s'} = ( \mbf{v'}, \mbf{h'} )$, we define the vertex switching operation, $G_{\mbf{s}\mbf{s'}}: \mbb{R}^{n\times m} \mapsto \mbb{R}^{n\times m}$, on the weight matrix,
\begin{equation*}
G_{\mbf{s} \mbf{s'}}(\mbf{W}) = \mbf{W'},
\end{equation*}
such that for $\forall i,j$,
\begin{equation*}
W_{ij}v'_i h'_j = W'_{ij}v_ih_j.
\end{equation*}
\end{definition}

\begin{remark}
It is easy to verify that the switching operation is symmetric with respect to its two subscripts, or
\begin{equation*}
G_{\mbf{s}\mbf{s'}} = G_{\mbf{s'}\mbf{s}}.
\end{equation*}
Furthermore, the operations form an abelian group action on $\mbf{W}$, or more specifically,
\begin{equation*}
\forall \mbf{s}, \mbf{s'}, \mbf{s''}, \quad 
G_{\mbf{s}\mbf{s''}} = G_{\mbf{s}\mbf{s'}}G_{\mbf{s'}\mbf{s''}} = G_{\mbf{s'}\mbf{s''}}G_{\mbf{s}\mbf{s'}}.
\end{equation*}

Given two states, $\mbf{s}$ and $\mbf{s'}$, it is convenient for us to refer to the set of matrix indices, $F = \{ (i,j) \cond v'_ih'_j = -v_ih_j \}$, as the switching subset from $\mbf{s}$ to $\mbf{s'}$. An alternative construction of $F$ is given in Appendix \ref{appendix:metric_proof}. The vertex switching operation on each weight element can then be defined uniquely as
\begin{equation*}
G_{\mbf{s}\mbf{s'}}(W_{ij}) = 
\begin{cases}
-W_{ij} \quad &\text{if} \quad (i,j) \in F, \\
W_{ij} &\text{if} \quad (i,j) \notin F^c,
\end{cases}
\end{equation*}
where $F^c = \{ (i,j) \cond (i,j) \notin F \}$. In some sense, the switching subset $F(\mbf{s},\mbf{s'})$ defines the ``transition" from state $\mbf{s}$ to state $\mbf{s'}$ (or the inverse transition). A visual representation of the switching subset is given in Fig. \ref{flip_fig}. \\
\end{remark}

Given $\mbf{s}$ and $\mbf{s'}$, we let $\mbf{W'} = G_{\mbf{s}\mbf{s'}}(\mbf{W})$ be the transformed weight matrix under the vertex switching operation. If we denote $E$ as the energy function of the RBM with weights $\mbf{W}$, and $E'$ as the energy function with weights $\mbf{W'}$, then it is clear that $E(\mbf{s'}) = E'(\mbf{s})$. Note that the energy difference between the two states is given as
\begin{equation}
\label{eng_diff}
E(\mbf{s'}) - E(\mbf{s}) 
= E'(\mbf{s}) - E(\mbf{s})
= 2\sum_{F(\mbf{s}, \mbf{s'})} W_{ij}v_ih_j,
\end{equation}
noting the weight elements are summed over the switching subset $F(\mbf{s}, \mbf{s'})$. 

\subsection{Gauge Fixing the RBM}
\label{gauged_rbm}

In this work, it is convenient to let $\mbf{s}$ be the ground state of the RBM and $\mbf{s'} = \mbf{+1}$, then we refer to $\mbf{W'} = G_{\mbf{s}\mbf{1}}(\mbf{W})$ as the corresponding weight matrix that has been {\it gauge fixed} \cite{lgt}, such that $E'(\mbf{+1}) = E(\mbf{s})$ is the ground state energy. We shall refer to RBM with weights $\mbf{W'}$ that has been gauged fixed as a {\it gauged} RBM; an equivalent definition of a gauged RBM is given as follows.

\begin{definition}
An RBM is said to be gauged if its ground state is $\mbf{+1}$.
\end{definition}

\begin{remark}
Note that given any state $\mbf{s}$ and weight matrix $\mbf{W}$ with ground state $\mbf{s_0}$, the state can always be expressed equivalently in terms of a gauged RBM. To see this, we first realize that
\begin{equation*}
G_{\mbf{1}\mbf{s}}(\mbf{W}) = G_{\mbf{s}\mbf{s_0}} \circ G_{\mbf{1}\mbf{s_0}}(\mbf{W}) = G_{\mbf{s}\mbf{s_0}}(\mbf{W'}),
\end{equation*}
where $\mbf{W'}$ is the gauged weight matrix. And in the gauged RBM, the state can be expressed as the transition from the ground state via the switching subset $F(\mbf{s_0}, \mbf{s})$. See Fig. \ref{gauge_fig} for a visual representation of how a $2 \times 2$ RBM is gauged.
\end{remark}

\begin{figure}
\begin{center}
\includegraphics[scale=0.6]{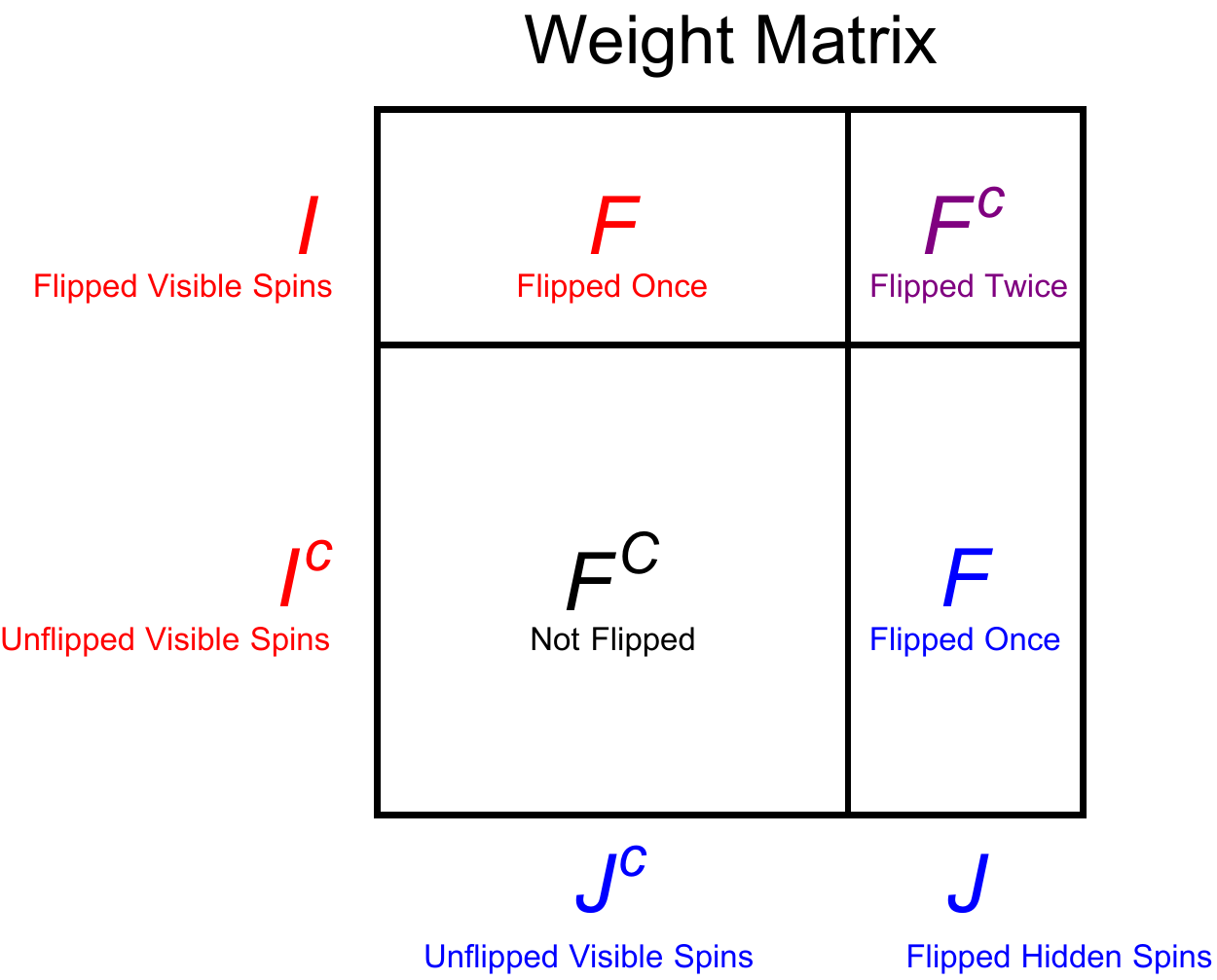}
\end{center}
\caption{Vertex switching illustrated in terms of the weight matrix. The elements in the lower-left block are not flipped. The elements in the upper-left and lower-right blocks are flipped once. The elements in the upper-right block are flipped twice, hence remain the same.}
\label{flip_fig}
\end{figure}

\begin{figure}
\begin{center}
\includegraphics[scale=0.5]{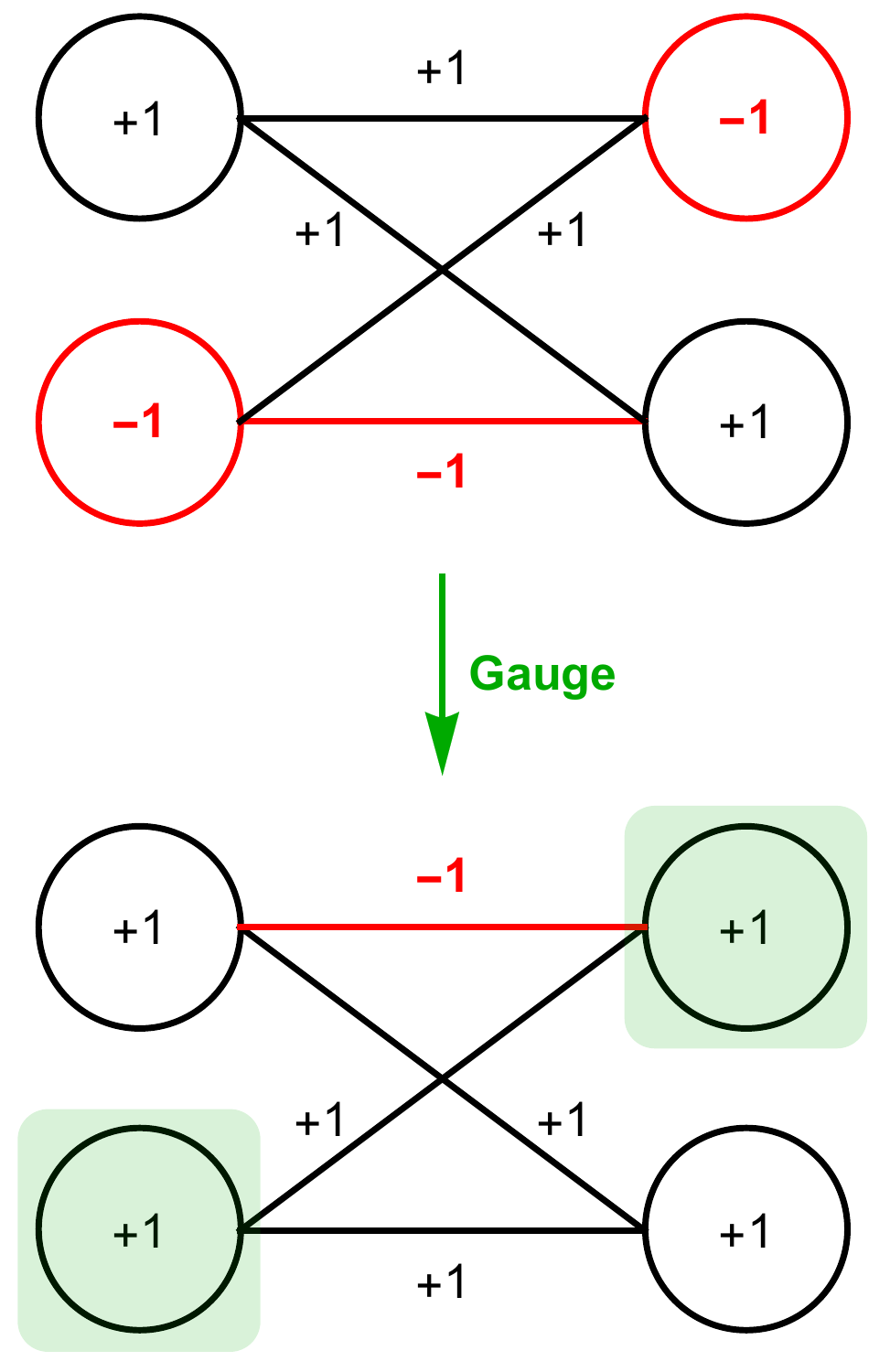}
\end{center}
\caption{Gauging a $2\times 2$ RBM; note that the new ground is the $\mbf{+1}$ state, and every edge with only one end connected to a flipped spin is negated. }
\label{gauge_fig}
\end{figure}

We continue the discussion on certain properties of a gauged RBM in this subsection. We first note that for a gauged RBM with weights $\mathbf{W}$, the ground state energy is simply the sum of all weights
\begin{equation*}
E(\mathbf{+1}) = -\sum_{ij} W_{ij}v_ih_j = -\sum_{ij}W_{ij}(+1)(+1) = -\sum_{ij}W_{ij}.
\end{equation*}
From Eq.~(\ref{eng_diff}) we see that any other state $\mathbf{s}$ can be simply expressed as
\begin{equation*}
E(\mathbf{s}) = E(\mathbf{+1})+2\sum_{F}W_{ij}.
\end{equation*}
We then see that the energy gap between an arbitrary state and the ground state is twice the sum of the weight elements in the corresponding switching subset. Note that for a gauged RBM, the sum of the weight elements of {\it any} switching subset must be positive, otherwise we obtain a configuration with a lower energy than the ground state configuration, thus contradicting the definition of a gauged RBM. We call this the {\it positive-sum condition}, and it can be used to check whether an RBM is gauged or not.

In practice, it is convenient to first generate the gauged RBM weight matrix by ensuring that the ground state is kept at $\mathbf{+1}$ (see Section \ref{frus_loop}), and then the weight matrix can be later processed to have any given spin configuration $\mathbf{s_0}$ as the ground state by simply taking the inverse gauge transformation. In fact, this procedure can be performed on any CSP problem on a general graph structure that possesses local gauge symmetry.

As a historical sidenote, this procedure is referred to as {\it gauge fixing} in statistical mechanics \cite{lgt}, which is common in the study of gauge fields. For the Ising model on a 2D lattice, there is a polynomial time algorithm for gauge fixing the lattice, as the problem can be formulated equivalently on the dual lattice as minimizing the $L1$ norms of strings \cite{disorder,barahona}. In fact, the problem of gauge fixing is in fact equivalent to finding the ground state energy of the Ising spin glass, and becomes an NP-hard problem for any non-planar graph \cite{istrail}. This means that the RBM cannot be gauge fixed efficiently unless the ground state configuration is known beforehand. In other words, we are essentially gauge fixing the RBM prior to generating the instance, and ensure that the gauge is not broken in the generation process. Instances resulting from such generation procedures are referred to as {\it planted} instances \cite{feldman}.

\subsection{Metric}
\label{metric}

It is generally useful to define a metric \cite{intro_topo} on the spin configuration space of the RBM, so we have a sense of how ``different" or ``far apart" two spin states are. A na\"ive choice would be to count the number of spins that are different between the two states, which is proportional to the $L_1$ distance or the so-called ``Manhattan distance'' \cite{taxi}. This metric is, however, not particularly insightful if the underlying graph structure is bipartite. To see this most clearly, we can flip all the spins and arrive at a state that is maximally far from the original state. However, since an unbiased RBM possesses $\mbb{Z}_2$ symmetry flipping all the spins will leave the energy invariant. Therefore, given two states, a more useful distance would be based on the corresponding switching subset, or $|F(\mbf{s}, \mbf{s'})|$.

If we denote the number of visible spins different between the two states as $n'$ and number of hidden spins different between the two states as $m'$, then we can write
\begin{equation*}
|F(\mbf{s},\mbf{s'})| = nm' + n'm - 2n'm'.
\end{equation*}
To normalize this quantity such that it ranges from $0$ to $1$, we define the distance as
\begin{equation*}
d(\mbf{s}, \mbf{s'}) = \frac{|F(\mbf{s},\mbf{s'})|}{nm}.
\end{equation*}
Note that $d$ generates a pseudometric space in which two states are distinguishable up to a global spin flip, which is a desirable property. A detailed discussion of this distance is given in Appendix \ref{appendix:metric_proof}.

To see the usefulness of this metric, we first assume all the weight elements, $\mathbf{W}$, and spins, $\mathbf{s}$, are iid (identically and independently distributed), where $W_{ij}$ is normal with mean $\mu$ and variance $\sigma^2$, and $v_i$ or $h_j$ has equal probabilities to be $+1$ or $-1$. We can then show that (see Appendix \ref{appendix:variance})
\begin{equation*}
\text{Var}_{ \{ \mbf{W},\mbf{s},\mbf{s'} \}}\big( E(\mathbf{s'})-E(\mathbf{s}) \bcond d(\mbf{s},\mbf{s'})=d \big) = 4nmd(\mu^2+\sigma^2),
\end{equation*}
so we see that the variance in the energy gap between two random states of a random RBM is proportional to the distance between the two states.

\subsection{Frustration}
\label{ffrus}

The {\it frustration index} of a weighted graph is defined at the ground state to be the ratio between the sum of the magnitudes of the unsatisfied bonds and the sum of the magnitudes of all the bonds \cite{frus}. For a gauged RBM, it is simply the ratio between the sum of the absolute values of all negative weight elements and the sum of the absolute values of all weight elements
\begin{equation}
\label{frus_eq}
f = \frac{-\sum_{W_{ij}<0}W_{ij}}{\sum |W_{ij}|} = \frac{1}{2}\frac{\sum |W_{ij}| - \sum W_{ij}}{\sum |W_{ij}|}.
\end{equation}
For a randomly generated RBM, the gaps of the low energy states are closely related to the frustration index, both in terms of distance and energy. To see this directly, we consider the gap between the ground state and the first excited state (the state with the second lowest energy). Clearly, the weight elements in the switching subset of the first excited state must produce the smallest (positive) sum out of all the switching subsets, and a switching subset with a greater distance is in general more entropically favored. With these two observations in mind, we now focus on the two regimes of the frustration index, small and large.

In the regime of small frustration, we can, for the sake of argument, assume that $f=0$, meaning that for a gauged RBM, all weights are non-negative. In this case, a farther distance from the ground state would imply that the cardinality of the switching subset is larger, corresponding to a larger number of weight elements to be summed over. And since all weights are non-negative, this would mean that it is less likely for the sum to be the smallest out of all switching subsets. In other words, the entropic favoritism of a farther distance is suppressed by the requirement of having the smallest sum, so the distance between the ground state and the first excited state is usually close for an RBM with small frustration.

In the regime of larger frustration, however, this suppression does not occur, which is due to the abundance of negative weights, making it likely for a switching subset of a large cardinality to still produce a small sum through the cancellation of positive and negative weights. It is then probable that the distance between the ground state and the first excited state to be far for a highly frustrated RBM. This is, in fact, a general feature of spin glasses with typical graph structures \cite{Mezard}.

\subsection{Maximum Frustration}

One interesting question (which as far as we know is open) is what the upper bound to the frustration of an RBM of a given size is. This question is not only of mathematical interest, but also of practical importance, since knowing the maximum frustration provides a reference for evaluating the hardness of the generated instances to the maximal hardness in terms of the frustration index.

If we consider the RBM instance to be the underlying bipartite graph of an arbitrarily large quasirandom graph \cite{quasi} (such as a large Paley graph \cite{paley}), then it can be shown that there is no constant upper bound to the frustration index except for the trivial $0.5$\footnote{If the frustration index is larger than $0.5$, then the sum of all weight elements of the gauged RBM must be negative. This means that we can take the switching subset to be the entire matrix, and break the positive-sum condition.}. This can be shown as a corollary of the fourth theorem stated and proved in Chung's paper \cite{quasi}, which we do not state here as it is not directly relevant to our work.

Note that the argument holds only if the RBM is assumed to be arbitrarily large, which is generally unrealistic in the context of training an RBM. In Appendix \ref{appendix:2m}, we then consider the special case of a $2 \times m$ RBM, and show that the frustration index is bounded above by $0.25$, which is the maximum frustration that the loop algorithm can generate (see Section \ref{atom}).

\subsection{Local Minima}

For the purpose of this work, we define a {\it local minimum} to be an RBM state which does not yield a lower energy by flipping any one spin. This means that if a local optimization algorithm (which operates through single spin-flips) were to arrive at this state, it will essentially become ``stuck" without the aid of any stochasticity. Therefore, we see that the population of local minima of a generated instance is an important characterization of its hardness, as it relates to the likeliness for a local optimizer descending in its energy landscape to be trapped, thus failing to discover the ground state (or the global minimum). 

In terms of the weight matrix of an RBM, it is easy to see that if we gauge a local minimum to be the $\mbf{+1}$ state, then the sum of any row or any column of weight elements must be non-negative, which is a weaker condition than the positive sum condition for the global minimum (see Section \ref{gauged_rbm}). This is because the ground state must be a local minimum, but a local minimum is not necessarily the ground state.

\section{Random Frustrated-Loop Algorithm}
\label{frus_loop_algo}

The frustrated-loop algorithm in its original form was proposed to generate spin-glass problems with known solutions to test the performance of the D-Wave quantum annealer \cite{loop}. Although the algorithm was initially proposed to be implemented on a 3D square lattice, it can be easily generalized to any graph as long as it is not acyclic \cite{graph}. For the purpose of this paper, we focus on the case of a complete weighted bipartite graph \cite{graph}, which, as shown in Section \ref{bi_conv}, is general enough to represent a weighted MAX-2-SAT problem with any underlying graph structure. 

Due to the non-local connectivity of a complete bipartite graph, the algorithm in its original form is unable to generate sufficiently hard instances at high loop densities (see Section \ref{limit}), so we propose a modified version of this algorithm (see Section \ref{modi_algo}) that will be more suitable for the generating hard instances on a complete bipartite graph. Before doing that, however, we need to define a few concepts and recall how the original frustrated-loop algorithm is implemented \cite{loop}. 

\subsection{Cycle}

In the language of graph theory, we can define a ``loop" simply as a closed path on a graph with non-repeating edges or vertices (a cycle) \cite{graph}. On a bipartite graph, the length of a loop must clearly be even, which we can denote as $2l$ (where $l \geq 2$ is a positive integer). We again resort to using the terminologies of RBM, and refer to vertices in one set as {\it visible nodes} and vertices in the other set as {\it hidden nodes}. 

To generate a random loop on an RBM, we start from a random visible node $i_1$, then ``walk" to a random hidden node $j_1$, then return to the visible layer on another random visible node $i_2$, and so on. Note that for every iteration, the node selected must not be already within the path, and the last iteration is a walk from $j_l$ back to the starting node, $i_1$, to ``close the loop". This cycle can be compactly expressed as
\begin{equation*}
i_1 - j_1 - i_2 - j_2 - ... - i_l - j_l - i_1.
\end{equation*}

From an algorithmic standpoint, generating this loop involves very little computational overhead as we can simply select $l$ nodes from the visible layer and $l$ nodes from the hidden layer in some random order, and simply connect them based on the order that they are chosen. Each edge in the loop can be assigned a weight as it corresponds to a two-body interaction between some visible spin and a hidden spin.

\subsection{Frustrated Loop}
\label{frus_loop}

We provide here a brief discussion of the random frustrated-loop algorithm in its original form \cite{loop}. The purpose of the algorithm is to generate an RBM instance with the ground state being $\mathbf{s} = \mathbf{+1}$. Trivially, one can set all the weights to positive. However, as discussed previously, this will result in an instance with zero frustration, or a ferromagnetic instance, and it will be extremely easy to solve. To make the instance non-trivial, we have to intentionally introduce negative weights in such a way such the ground state configuration $\mathbf{+1}$ is kept invariant, so we do not lose track of the planted solution.

One way of doing so is to generate a loop of length $2l$ and set all the edge weights in the loop to $+1$ except for a single edge weight which we set to $-1$. Then it can be checked that the ground state energy of this RBM subsystem is $E=-(2l-1)+1=-2(l-1)$, with a $2l$-fold degeneracy each corresponding to one of the $2l$ bonds to violate in this loop, with $\mbf{+1}$ being one of the ground state violating the negative weight. 

Now, we can drop multiple loops on the graph and ``sum" them together, meaning that when an edge is shared by multiple loops, we sum the weight contributions from the multiple loops at that given edge. And if an edge is not a part of any loop, then we set the corresponding weight to $0$. It is shown that this procedure leaves the ground state $\mathbf{+1}$ invariant \cite{loop}. 

Note that there is no obvious way to directly control the frustration index of an instance generated this way. As we shall see in later subsections the frustration index is affected by the lengths of the loops, and the way that the loops intersect. To address this problem, we can fix the loop length to be the smallest possible value, so that the frustration contribution of a single loop is fixed (see Section \ref{atom}). Second, we relax the condition that the negative weight in the loop has to be $-1$, which enables direct tuning of the frustration index (see Section \ref{tune}). Lastly, we prohibit the negative weights from overlapping with the positive weights, so that the frustration index does not decrease with increasing loop density (see Section \ref{int}). 

\subsection{Loop Atom}
\label{atom}

In the original algorithm, only loops above a certain length are kept \cite{loop}, because it is argued that small loops may lead to excessively difficult instances \cite{loop}. The reason for this is because smaller loops contribute a larger frustration. To see this, consider a loop of length $2l$, then we see that the frustration index of this subsystem is given as (see Eq.~(\ref{frus_eq}))
\begin{equation*}
f = \frac{1}{2} \Big( 1-\frac{2(l-1)}{2l} \Big) = \frac{1}{2l},
\end{equation*}
implying that the frustration index scales inversely with the loop length. To keep the frustration index in control, we then fix the lengths of all loops to the smallest value, which is $4$. The instances generated by these small loops are then expected to be hard, so we allow the hardness to be tuned down by altering the magnitude of the negative weight (see Section \ref{tune}). 

From here on, we refer to a frustrated loop of the smallest length as a {\it loop atom}. Note that we are not sacrificing any generality by restricting the length of the loops to $4$, as a frustrated loop of any given length can be decomposed as the sum of loop atoms. To see this, we consider a loop of length $2l$ which we wish to decompose, which we denote as $i_1 - j_1 - ... - j_l - i_1$. Without loss of generality (WLOG), we assume that the edge with the negative weight is $j_l - i_1$. Now, let a smaller loop of length $2(l-1)$ be $i_1 - j_1 - ... - j_{l-1} - i_1$, with the negative edge being $j_{l-1} - i_1$, and let a loop atom be $i_1 - j_{l-1} - i_l - j_l - i_1$ with the negative edge being $j_l - i_1$. Note that the two loops intersect at $i_1 - j_{l-1}$, with the contributions from the two loops canceling out, resulting in the loop $i_1 - j_1 - ... - j_l - i_1$ with the negative edge being $j_1 - i_1$, which is simply the original loop that we wished to decompose. We can then repeat the decomposition on the length $2(l-1)$ loop into a loop of length $2(l-2)$ plus a loop atom, and reiterate this procedure until the length of the loop shrinks to $4$. Thus, we see that a frustrated loop of length $2l$ can be decomposed into $l-1$ loop atoms.

We then see that the loop atom is general enough to produce a frustrated loop of any given length, and from here on, we shall always assume that an RBM instance is generated only with loop atoms. 

\begin{figure}
\begin{center}
\includegraphics[scale=0.5]{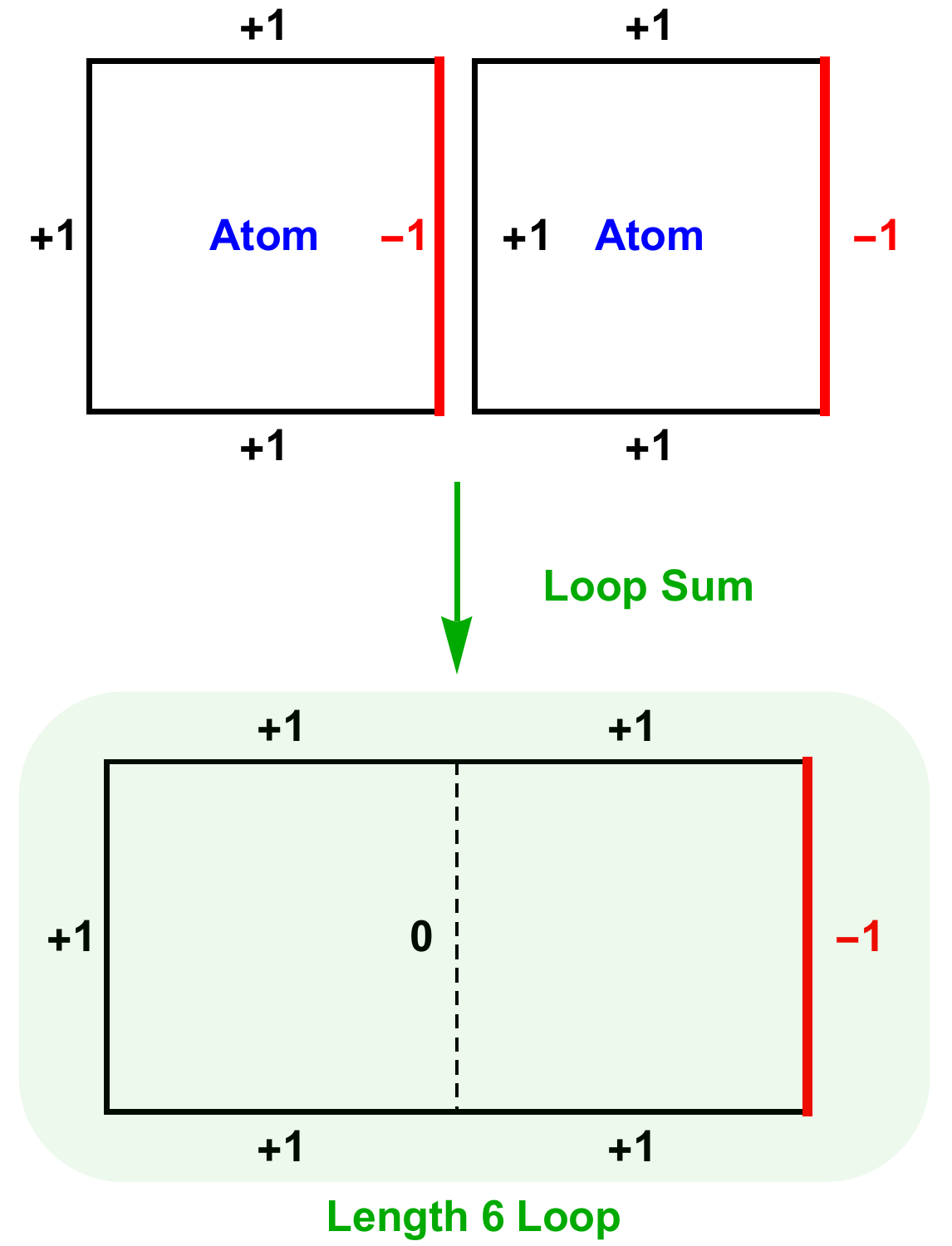}
\caption{Summing two-loop atoms results in a larger frustrated loop of length six. Note that the frustrated edge from the left loop and the positive edge from the right loop cancel out.}
\end{center}
\end{figure}

\subsection{Tunable Frustration}
\label{tune}

Recall that the frustration index of a single loop atom is $0.25$, which may give rise to excessively difficult instances. To allow for tunability of the frustration index, we can relax the condition that negative edge weight in the loop atom must be $-1$. Instead, we can set the negative weight to be $-\alpha \in [-1,0]$, which guarantees the satisfaction of the positive sum condition (see Section \ref{gauged_rbm}). The frustration index contributed by this loop is then given as
\begin{equation*}
f = \frac{\alpha}{3+\alpha},
\end{equation*}
with $f$ increasing from $0$ to $0.25$ as we increase $\alpha$ from $0$ to $1$. The ground state energy of an instance generated by $N$ of such loops is then $-N(3-\alpha)$, which we can use to verify the correctness of the solver.

\subsection{Intersection}
\label{int}

Often times, two or more loops will ``intersect", meaning that they will share one or more edges. They can either intersect constructively, meaning that all the weight contributions at the intersected edge are of the same sign, or intersect destructively, meaning that not all weight contributions are of the same sign. If we assume that the negative weight of the loop atom is $-1$, then each loop contributes $-(1+1+1-1)=-2$ units of energy to the ground state energy, so if we have $N$ such loop atoms, then the ground state energy is simply $-2N$. Therefore, the expression for the frustration index in Eq.~(\ref{frus_eq}) can be reduced to
\begin{equation*}
f = \frac{1}{2} - \frac{N}{\sum |W|}.
\end{equation*}
If none of the loops intersect or if the they only intersect constructively, then it is easy to see that $\sum |W| = 4N$, resulting in a frustration index of $f=0.25$, or the maximum frustration that one can generate using loop atoms.

Although constructive intersections do not change the frustration index, destructive intersections decrease the frustration index. From here on, the term ``intersection" is used to refer solely to destruction intersection. To see how the frustration index is affected by intersections, we consider the scenario where we drop a new loop atom randomly on a graph with existing loops, then whenever an edge weight of the new loop opposes the sign of the existing weight on that edge, we denote this occurrence as an {\it intersection event}. 

If we denote the number of intersection events on the RBM as $N_{\times}$, then the frustration index is given as
\begin{equation*}
f = \frac{1}{2}(1-\frac{N}{2N - N_{\times}}).
\end{equation*}
From Appendix \ref{appendix:poisson}, we see that the expected number of intersection events, $\E(N_{\times})$, starts from $0$ and approaches $N$ asymptotically from below as the number of loops increases. If we were to approximate the expected value of the frustration index as
\begin{equation*}
\mathbf{E}(f) = \frac{1}{2}(1 - \frac{N}{2N - \E(N_{\times})}),
\end{equation*}
then it is clear that the frustration index starts from $0.25$ and approaches $0$ as $N$ increases, meaning that the frustration index decreases with increasing density of loops. 

Intuitively, this makes sense. If the loop density is small, then it is unlikely that the loops will intersect, so the frustration will remain at $0.25$. As we increase the loop density, intersection events become more common, and the frustration index will decrease due to cancellation of positive and negative weights. If the loop density is sufficiently high, then the law of large numbers will guarantee that all weights will be positive as positive contributions are three times as dominant than negative contributions, and the frustration index will become zero. In short, the frustration index is inversely correlated with the loop density in a random manner, so to prevent any decay and uncertainty in the frustration index, it is necessary that we prohibit any form of destructive intersections.

\subsection{Generality}
\label{gene}

An interesting question regarding the generality of the loop algorithm is whether it is possible to generate any RBM instance (which WLOG we here assume to be  gauged) of frustration $f\leq 0.25$ with only loop atoms\footnote{Note that it is inconsequential whether the negative weight of the loop atom $-\alpha$ is tunable or fixed at $-1$, as a tunable loop atom can always be expressed as some conical combination of four loop atoms with $-1$ negative weight.}. Before we extend on this discussion, we first point out that it is trivially possible to generate any ferromagnetic instance (with all weights being positive) using loop atoms, so the question can be formulated alternatively as whether any gauged RBM matrix with $f\leq 0.25$ can be expressed as a conical combination of loop atoms and a non-negative matrix. A conical combination is essentially a linear combination with all the coefficients being non-negative \cite{convex}, where the non-negative coefficient condition is crucial here because we cannot ``negate" the sign of a loop atom as it will not guarantee invariance of the planted solution.

We now make the statement of this problem more concise. We can consider the weight matrix $\mathbf{w}$ of an unbiased $n\times m$ RBM to be an $nm$ dimensional vector. A loop atom $\mathbf{l}_c$ can also be considered a vector in the same vector space, with most of its components being zero. It can be checked that there are $k=4{n\choose 2}{m\choose 2}$ possible loop atoms. Then the problem becomes whether we can find a set of $K$ non-negative numbers $\{x_1,x_2,...,x_k\}$ such that the following is true
\begin{equation*}
\sum_{k=1}^{K} x_k \mathbf{l}_k
\leq \mathbf{w}.
\end{equation*}
If we denote $\mathbf{L}=[\mathbf{l_1}, \mathbf{l_2}, ... ,\mathbf{l_k}]$, then the question can be made even more concise:
\begin{equation}
\label{f1}
\text{Does the system } \mathbf{L}\mathbf{x}\leq \mathbf{w} \text{ have a solution with } \mathbf{x}\geq 0 \text{ ?}
\end{equation}

This is a system of inequalities, with its dual problem given as follows \cite{linear}:
\begin{equation}
\label{f2}
\begin{split}
&\text{Does the system } \mathbf{L}^T\mathbf{y}\geq 0 \text{ have a solution with } \\
&\mathbf{w}^T\mathbf{y}<0 \text{ and } \mathbf{y}\geq 0?
\end{split}
\end{equation}
Exactly one of the two statements, (\ref{f1}) and (\ref{f2}) above, can be true at a given instance, so proving statement (\ref{f1}) true is equivalent to proving statement (\ref{f2}) false. In Appendix \ref{appendix:gen}, we explicitly show that statement (\ref{f2}) is false for a $2\times 3$ RBM. The result can be easily generalized to a $2\times m$ RBM. The general case of this question for an $n\times m$ RBM is (as far as we know) still open.

\subsection{Biases}

In Section \ref{ghost}, we showed that any bias term can be expressed as an interaction between a spin and some fixed spin $v_{n+1}$ or $h_{m+1}$. In previous sections, we limited our focus to only unbiased RBM instances, so the frustration contribution of the loop atom is restricted by the positive sum condition, under the assumption that all four spins in the loop atom can be freely flipped, generating $8$ possible switching subsets (up to a $\mbb{Z}_2$ symmetry). However, if the RBM instance is biased and the loop atom contains one or more fixed spins, then the number of switching subsets decreases, which relaxes the positive sum condition slightly, and this will result in a possible increase in the upper bound of the frustration contribution of the loop atom. 

If the loop atom contains one fixed spin, which WLOG we assume to be $h_{m+1}$, then we can write the loop atom as $i_1-j_1-i_2-(m+1)-i_1$, with the two visible biases generated by this loop being $W_{i_1,m+1}$ and $W_{i_2,m+1}$. At first glance, the spin $h_{m+1}$ being held fixed seems to imply that the number of switching subsets decreases. This observation is however not true as the spin $h_{m+1}$ can be flipped indirectly (up to a $\mbb{Z}_2$ symmetry) through simultaneously flipping the spins $v_{i_1}$, $v_{i_2}$ and $h_{j_1}$. Therefore, the restrictions on the weights of the loop will not change (see Section \ref{tune}), and the frustration contribution is still bounded above by $0.25$.

If the loop atom contains two fixed spins, we can express the loop atom as $i_1-(m+1)-(n+1)-j_1-i_1$, which generates a visible bias, $W_{i_1,m+1}$, and a hidden bias, $W_{n+1,j_1}$. The element $W_{n+1,m+1}$ denotes a constant offset (independent of the spin states) and can be thus disregarded. In this case, disregarding the null switching subset, we have only $2^2-1=3$ possible switching subsets, and a possible weight assignment that satisfies the ``positive-sum condition" can be $W_{i_1,m+1}=W_{n+1,j_1}=1$ and $W_{i_1,j_1}=-\alpha$ for $0\leq \alpha\leq 1$. The maximum frustration contribution of this loop atom (ignoring the weight $W_{n+1,m+1}$) is then
\begin{equation*}
f_{max} = \max_{\alpha}\Big( \frac{\alpha}{1+1+\alpha} \Big) = \frac{1}{3},
\end{equation*}
which is greater than the previous upper bound of $0.25$. This relaxed upper bound for the frustration contribution of a loop atom to a biased RBM instance can be exploited to generate instances with greater variation in hardness. 

\subsection{Extension to General Graphs}
\label{ext}

Our choice to study the frustrated-loop algorithm on a bipartite graph is due to the ease of theoretical analysis, and its direct application to RBM pre-training (see Section \ref{app}). However, the algorithm can be easily applied to any connected graph that is not acyclic \cite{graph}. One simply has to detect a sufficient number of random cycles on the graph, and generate a frustrated loop on each cycle by setting one of its edges to $-1$ and the rest to $+1$. 

An efficient way for finding all the cycles in a graph is well known. We first begin by finding a cycle basis of the graph, or the minimal set of cycles from which all cycles can be generated through the symmetric difference operation \cite{boo}. The standard way to find a cycle basis is from the spanning tree of the graph, and many refined algorithms already exist for this purpose \cite{span,cycle_1,cycle_2,cycle_3}. After finding the cycle basis, we then take the symmetric difference between two or more randomly selected basis cycles to generate a new random cycle \cite{enum}.

A cycle on a general graph structure can, in general, have any length greater than $3$. Note that a length-3 loop has a frustration contribution of $\frac{1}{3}$ (in contrast to the maximal frustration of $\frac{1}{4}$ for a bipartite graph), meaning that it is possible to generate instances of even higher frustration on a general graph than on a bipartite graph.

In terms of machine learning, this means that the frustrated loop algorithm can be applied to a variety of neural network structures. For example, it can be applied to a deep neural network \cite{deep} which can be described as a $k$-partite graph, or a fully connected Boltzmann machine which can be described by a complete graph \cite{fully_boltz}.

\subsection{Algorithm}
\label{secfrusalgo}

A simple version of the frustrated-loop algorithm pseudocode is given in Algorithm \ref{frusalgo}. The code allows for the basic functionality of independent tuning of the frustration index and the loop density. As reasoned in Section \ref{frus_loop}, this algorithm prohibits destructive interference events and only uses loop atoms with the negative weight of the loop being tunable. 

Note that variations on the code can be made, depending on the purpose of the test. Some examples are: the edge weights can be made normal random variables with small standard deviations to introduce more randomness; constructive interference can be also prohibited to have more consistent testing results; the bias terms can be intentionally made larger to generate more difficult instances. We only show the basic version here to avoid unnecessary complications.

\begin{algorithm}[H]
\caption{Random Frustrated-Loop Algorithm}
\label{frusalgo}
\begin{algorithmic}[1]

\State $\text{Initialize an empty } n\times m \text{ matrix } \mathbf{W}$
\State $\alpha = 3f/(1-f)$
\For{$iteration \in [[1,N_{loops}]]$} 
\State $\text{Choose a random column } j_1$
\State $\text{Choose two random rows } i_1,i_2 \text{ such that} \newline \hspace*{3em} W_{i_1j_1}\geq 0 \land W_{i_2j_1}\leq 0$
\State $\text{Choose another random column } j_2 \text{ such that } \newline \hspace*{3em} W_{i_1j_2}\geq 0 \land W_{i_2j_2}\geq 0$
\State $W_{i_1j_1}\gets W_{i_1j_1}+1,\, W_{i_2j_1}\gets W_{i_2j_1}-\alpha$
\State $W_{i_1j_2}\gets W_{i_1j_2}+1,\, W_{i_2j_2}\gets W_{i_2j_2}+1$
\EndFor
\State $\text{Generate a random state vector } \mathbf{s}\in\{-1,1\}^{n+m}$
\State $\text{Gauge } \mathbf{W} \text{ such that } \mathbf{s} \text{ is the lowest energy state}$

\end{algorithmic}
\end{algorithm}

The algorithm can be easily modified to generate MAX-2-SAT instances of uniform weights and tunable clause density. We begin by setting the parameter $\alpha = 1$ for all the loop atoms, and prohibit constructive and destructive intersections of the loops altogether. In this case, it is clear that the absolute values of all non-zero weights will be $1$, and if we convert the $n\times m$ RBM instance into a MAX-2-SAT problem through the procedure described in Section \ref{inv_conv}, we will obtain a MAX-2-SAT instance with uniform weights. Since the loops do not intersect, we see that the number of non-zero weights is exactly $4N$, where $N$ is the number of loops. The clause density of the corresponding MAX-2-SAT instance is then given as
\begin{equation*}
\rho = \frac{8N}{nm},
\end{equation*}
meaning that the clause density can be directly controlled by the number of loops.

\subsection{Limitation}
\label{limit}

Even though the frustrated-loop algorithm is able to generate maximally frustrated instances, this does not necessarily imply that the instances are sufficiently hard. This is due to the randomness of the distribution of negative weights on the weight matrix, making it unfavorable for the population of local minima. To be more specific, the expected value of each weight element for a maximally frustrated instance (setting $\alpha=1$ for all loop atoms) generated is $1/2$ with a standard deviation of $\sqrt{3}/2$, and if we were to find the sum of a large subset of weights corresponding to a given switching subset $F$, the sum is then $\frac{1}{2}|F|\pm \frac{\sqrt{3}}{2}\sqrt{|F|}$, and this value is most likely positive for a large $|F|$ as a result of the LLN (Law of Large Numbers). Therefore, the sum of the elements in any row or column (corresponding to a visible or hidden spin respectively) is positive w.h.p. (with high probability), meaning that local algorithms based on single spin-flip updates will most likely flip all the spins to the $\mbf{+1}$ ground state after a small number of sweeps, making the instance incredibly easy. In Appendix \ref{appendix:local_max}, we provide a formal discussion on the absence of local minima for instances generated by the random algorithm at high loop density in the limit of large system size. To generate sufficiently hard instances at high density, it is then necessary to enforce certain structures on the loop atoms, such that the weight distribution generated by the loop atoms is favorable for the population of local minima.

\section{Structured Frustrated-Loop Algorithm}
\label{modi_algo}

\begin{figure}
\begin{center}
\includegraphics[scale=0.8]{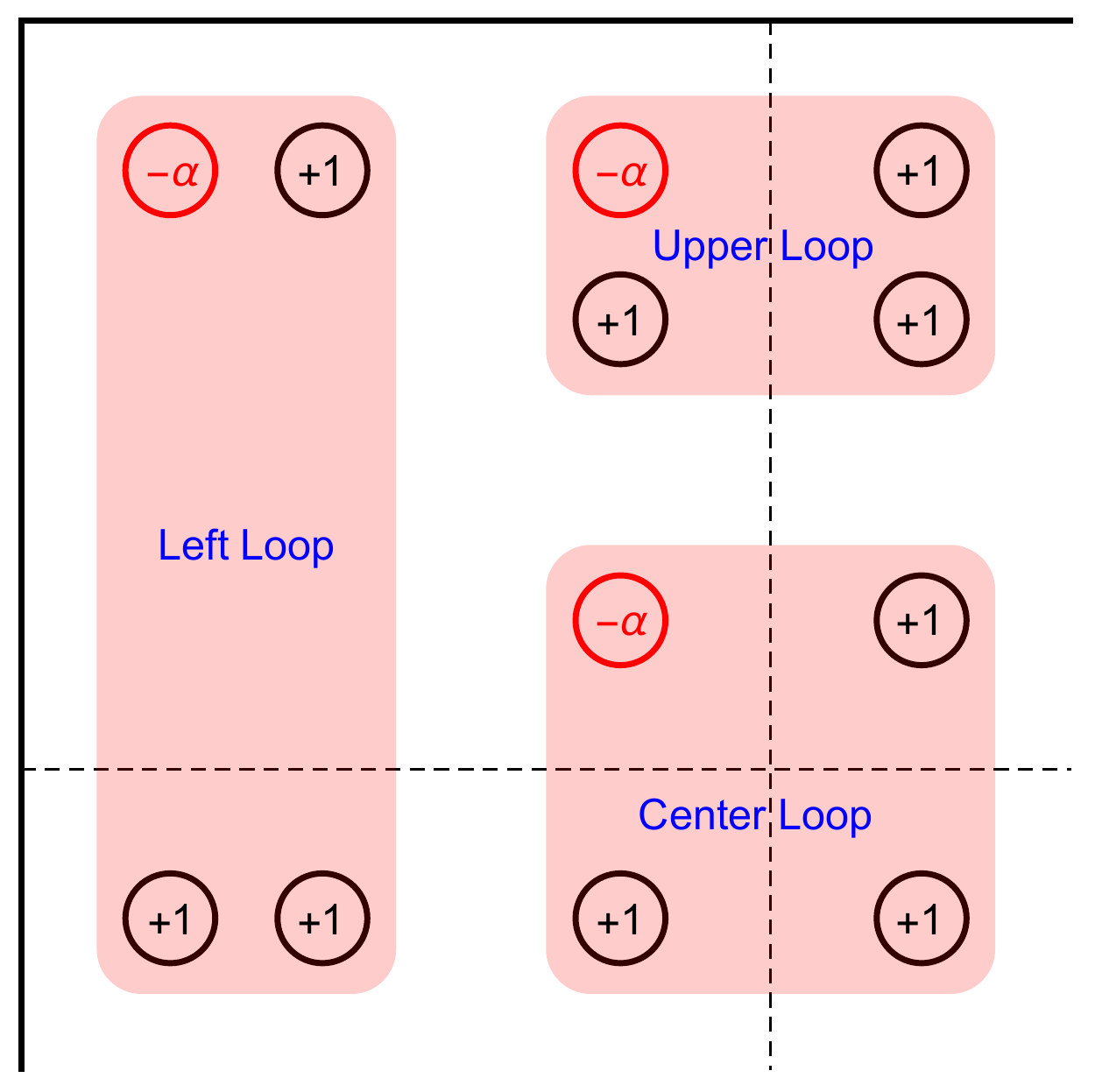}
\end{center}
\caption{The weight matrix is divided into four blocks by horizontal and vertical dashed lines. A left loop is on the left of the vertical line and crosses the horizontal line. An upper loop is above the horizontal line and crosses the vertical line. A center loop crosses the intersection of the two lines. Note that the negative weights are all located at the upper left block of the matrix.}
\label{mod}
\end{figure}

As mentioned in Section \ref{limit}, to generate sufficiently hard instances at high loop density, the loop atoms cannot be dropped on the weight matrix in a completely random fashion, otherwise local minima will fail to populate the energy landscape, resulting in instances that are incredibly easy. Therefore, we have to enforce certain conditions on the loop atoms for the generation of hard instances in the regime of high loop density.

\begin{algorithm}[H]
\caption{Structured Frustrated-Loop Algorithm}
\label{structurealg}
\begin{algorithmic}[2]

\State $\text{Initialize an empty } n \times m \text{ matrix } \mathbf{W}$
\State $\alpha = 3f/(1-f)$
\State $N_1 + N_2 + N_3 = N_{loops}$
\State $d \in (0,1]$
\State $n_1 = \lceil (n-1)d \rceil; \quad m_1 = \lceil (m-1)d \rceil$

\For{$iteration \in [[1,N_1]]$} 
\State $\text{Choose a random row } i_1 \text{ in } [[ \, 1, n_1 \, ]] \text{;}$
\State $\text{Choose a random row } i_2 \text{ in } [[ \, n_1+1 , n \, ]] \text{;}$
\State $\text{Choose two random columns } j_1,j_2 \text{ in } [[ \, 1, m_1 \, ]]$
\State $\text{such that:}$
\State $W_{i_1j_1}\leq 0 \text{, and } W_{i_1j_2}, W_{i_2j_1}, W_{i_2j_2}\geq 0$
\State $W_{i_1j_1}\gets W_{i_1j_1}-\alpha,\, W_{i_2j_1}\gets W_{i_2j_1}+1$
\State $W_{i_1j_2}\gets W_{i_1j_2}+1,\, W_{i_2j_2}\gets W_{i_2j_2}+1$
\EndFor

\For{$iteration \in [[1,N_2]]$} 
\State $\text{Choose two random rows } i_1,i_2 \text{ in } [[ \, 1 , n_1 \,]] \text{;}$
\State $\text{Choose a random column } j_1 \text{ in } [[ \, 1, m_1 \, ]] \text{;}$
\State $\text{Choose a random column } j_2 \text{ in } [[ \, m_1+1 , m \, ]]$
\State $\text{such that:}$
\State $\text{...}$
\EndFor

\For{$iteration \in [[1,N_3]]$} 
\State $\text{Choose a random row } i_1 \text{ in } [[ \, 1, n_1 ]] \text{;}$
\State $\text{Choose a random row } i_2 \text{ in } [[ \, n_1+1 , n \, ]] \text{;}$
\State $\text{Choose a random column } j_1 \text{ in } [[ \, 1 , m_1 \, ]] \text{;}$
\State $\text{Choose a random column } j_2 \text{ in } [[ \, m_1+1 , m \, ]]$
\State $\text{such that:}$
\State $\text{...}$
\EndFor

\State $\text{Generate a random state vector } \mathbf{s}\in\{-1,1\}^{n+m}$
\State $\text{Gauge } \mathbf{W} \text{ such that } \mathbf{s} \text{ is the lowest energy state}$

\end{algorithmic}
\end{algorithm}

\subsection{Algorithm}
\label{sec_slog}

We first present the algorithm (see Algorithm~\ref{structurealg}), followed by an explanation of the advantage of this algorithm over the random counterpart at high loop density. For the sake of consistency, we prohibit the loops from intersecting (destructively). We start by dividing the gauged RBM weight matrix into four blocks: the upper-left block $B_1$, the upper-right block $B_2$, the lower-left block $B_3$, and the lower-right block $B_4$. The sizes of the four blocks are parameterized by the factor $d \in (0,1]$, with the size of block $B_1$ given as
\begin{equation*}
\lceil\, (n-1)d \,\rceil \times \lceil\, (m-1)d \,\rceil,
\end{equation*}
and the sizes of the remaining blocks are in accordance with this. Note that this parameterization guarantees the non-triviality of the four blocks (meaning that no block is sized $0\times 0$). 

Note that a loop atom expressed on a weight matrix can be visualized as four elements that form the vertices of a rectangle. To be more specific, the cycle $i_1-j_1-i_2-j_2-i_1$ can be expressed as a weight matrix with the indices of its non-zero elements being $(i_1,j_1)$, $(i_1,j_2)$, $(i_2,j_1)$, and $(i_2,j_2)$, which can be connected to form a rectangle in the 2D Cartesian coordinate system. In this algorithm, we classify the loop atoms into the following three types (see Figure \ref{mod}):
\begin{itemize}
\item Left Loop: Two vertices of the loop atom must be in $B_1$, and the other two vertices must be in $B_3$.
\item Upper Loop: Two vertices of the loop atom must be in $B_1$, and the other two vertices must be in $B_2$.
\item Center Loop: Every block must contain a vertex of the loop atom.
\end{itemize}
Furthermore, we restrict all the negative weights to block $B_1$, which effectively ``concentrates" the negative weight elements into the upper left block. As we shall see shortly, this concentration of negative weights is favorable for generating hard instances. 

There are a couple of important points to make regarding the symmetry of this algorithm. Note that if all four blocks are of the same size, then the choice of which block to concentrate the negative weights to is completely arbitrary, since a matrix can always be permuted such that the upper-left block becomes any one of the four blocks, and it is even possible to ``spread" the negative weights throughout the matrix through permutation. Furthermore, if the negative weight of the loop atom is $-1$, then the parameterization of this algorithm with $d$ and $1-d$ are equivalent up to a gauge transformation via the switching subset $B_1 \cup B_4$. Under this transformation, all non-zero elements of $B_1$ become positive and those of $B_4$ become negative, meaning that the roles of $B_1$ and $B_4$ are effectively ``swapped". 

\subsubsection{Multiple Metastable Clusters}

As an important sidenote, we point out that under the structured loop algorithm, only one metastable cluster is guaranteed, with the remaining clusters arising as a result of the stochasticity of the generation method, and such clusters can not be controlled directly. In general, planting multiple metastable clusters is expected to be an incredibly difficult task itself, as it is almost equivalent to the problem of pre-training the RBM \cite{cd}, a task well-known to be intractable. 

To see this, we note that planting an RBM instances where multiple metastable clusters can be controlled is equivalent to constructing an RBM energy landscape such that the local minima of the RBM energy, $E(\mbf{v},\mbf{h})$, are realized at multiple fixed states. This is equivalent to performing a maximum likelihood estimation on the weights of the RBM such that certain states of the joint PMF, $p(\mbf{v},\mbf{h}) = e^{-E(\mbf{v},\mbf{h})}$, is maximized. If we were to plant metastable clusters near visible states coinciding with the data set, then the amplitudes of the marginal PMF, $p(\mbf{v})$, over the data set is also expected to be maximized, as the joint and marginal distribution of a random RBM is expected to be highly correlated. The last statement is highly non-trivial, and is explored in greater depth in another work \cite{mode_train}, where this marginal-joint correlation is utilized to efficiently pretrain an RBM.

\subsection{Degeneracy and Local Minima}
\label{degen}

Here, we consider loop atoms with three positive edges, $+1$, and one negative edge with weight $-\alpha \in [0,1]$ (see Section \ref{tune}). If we denote the numbers of left, upper, and center loops as $N_1$, $N_2$, and $N_3$ respectively, then it can be shown that
\begin{equation*}
\begin{split}
&\sum B_1 = (N_1+N_2)-\alpha (N_1+N_2+N_3), \\
&\sum B_2 = 2N_2 + N_3 \quad \sum B_3 = 2N_1 + N_3, \quad \sum B_4 = N_3.
\end{split}
\end{equation*}
If we choose $F=B_1\cup B_4$ to be the switching subset, then the sum of all its elements is
\begin{equation*}
\sum F = (1-\alpha)(N_1+N_2+N_3).
\end{equation*}
We see that if $\alpha =1$, then this sum is zero, meaning that the ground state is at least two-fold degenerate. In this case, we have two clusters of low-energy states, one concentrated near the planted ground state, and one concentrated distance $|F|/nm$ from the ground state, and the energy gap between the two clusters can be tuned through the parameter $\alpha$. We can then choose $\alpha$ to be a value arbitrarily close to $1$ (but not $1$ itself) so that the degeneracy of $F$ is broken slightly, which results in instances that are able to ``trap" a local solver in the cluster far away from the planted ground state. We shall refer to this cluster as the {\it metastable cluster}. 

In Appendix \ref{appendix:l_max}, we show that for sufficiently large number of center loops, $N_3$, the population of local minima is expected to be large. We also argue that for an instance class of sufficiently large frustration, increasing the size parameter, $d$, increases the distance between the two clusters, meaning that it is more difficult for the solver to escape the metastable cluster and find the planted solution. In the next Section (Section \ref{test}), we show empirically, using simulated annealing, that there is an increase in difficulty for the structured frustrated-loop algorithm over the random frustrated loop algorithm in the regime of high loop density and large frustration.

\subsection{Extension to General Graphs}

Similar to how a loop atom can be expressed on an RBM weight matrix as a rectangle, a frustrated loop of length 4 can be expressed also as a rectangle on the adjacency matrix \cite{graph} of a general graph. To be more specific, the cycle $i_1-i_2-i_3-i_4-i_1$ can be expressed as a rectangle with indexes $(i_1,i_2)$, $(i_2,i_3)$, $(i_3,i_4)$, and $(i_4,i_1)$. Since the graph is not necessarily bipartite, the visible-hidden ordering of the coordinate indices is not well-defined, meaning that the rectangle is not unique. Alternatively, the rectangle can also be $(i_1,i_4)$, $(i_4,i_3)$, $(i_3,i_2)$, and $(i_2,i_1)$, which is simply the ``transpose" of the first rectangle on the adjacency matrix. Therefore, to ensure the required symmetry of the adjacency matrix, we express the loop atom as the sum of the two rectangle representations. 

Note that similar to the RBM case, we can, WLOG, restrict all the negative elements in the top-left block of the matrix. With this in mind, we can check that the transpose of a left loop is simply an upper loop, and the transpose of a center loop is still a center loop, with the negative edge remaining invariantly in the top-left block. Therefore, we see that we can choose to ignore, WLOG, the upper loop, and generate half of the adjacency matrix only with left and center loops, then take the sum of the generated matrix and its transpose to form the full adjacency matrix. 

\section{Experimental Evaluation}
\label{test}

In this study, we focus on generating $n \times n$ RBM instances (equal number of visible and hidden spins) using the random-loop algorithm (see algorithm \ref{frusalgo}) and the structured-loop algorithm (see algorithm \ref{structurealg}). An RBM instance can be generated with three parameters: the size of the system $n$, the frustration index $f$ (see Section \ref{tune}), and the loop density $\rho$, where the loop density is defined as the ratio between the number of loops and the size of the system
\begin{equation*}
\rho = \frac{N}{n}.
\end{equation*}
To study the hardness of the generated instance, we use simulated annealing (SA) \cite{sa} as a powerful stochastic optimizer to solve the generated instances of different parameter triplets $\{n,f,\rho\}$, and we record the number of sweeps it takes for the solver to find the ground state (see Section \ref{meas_diff}). The SA algorithm performs directly on the problem in the original RBM form (see Section \ref{11}). For testing the performance of a general MAX-SAT solver, one can easily convert the problem into the corresponding MAX-2-SAT form.

The evaluation of the hardness of the instances can be roughly divided into three parts. In the first two parts, we mainly study how the hardness difficulty scales with $\{n,f,\rho\}$ with instances generated with the random-loop algorithm. In the first part, we study how the hardness varies with the loop density, and observe the expected easy-hard-easy transitions, or hardness peaks \cite{easy-hard-easy}, with the peak amplitudes and locations dependent on $n$. In the second part, we study how the hardness scales with $n$ for different frustration indexes $f$, and find that we can realize drastically different scaling behaviors for small changes in the frustration index, which may be related to the characteristic {\it phase transition} of the corresponding spin-glass model \cite{complex_phase}. In the third part, we perform a comparative analysis between the random- and structured-loop algorithms in their abilities to generate hard instances at high-loop density. For the structured-loop algorithm, we observe a second easy-hard transition beyond the first hardness peak, and also find a doubly-exponentially scaling difficulty improvement factor (over the random loop algorithm) with respect to the frustration index. All of the testing results can be easily reproduced by using the MATLAB script \path{main.m} which includes the functionality of generating instances with both the random and structured loop algorithm with user-defined parameters $\{n,f,\rho\}$.

\subsection{Measurement of Hardness}
\label{meas_diff}

The difficulty of a generated instance is measured by the number of sweeps it takes for the simulated annealing (SA) algorithm (see Appendix \ref{appendix:simu}) summed over all runs, $N_{tot}$, to find the ground state configuration. A {\it sweep} is defined as an update over all the spins in the RBM, and a {\it run} is defined as a sequence of sweeps before the spin configuration is reset. The spins are reset if the solver fails to find the planted solution within the maximum allocated number of sweeps in a given run, which we denote as $N_{sweep}$. For each generated instance, the solver is run on a single core of an AMD EPYC 7401 24-core processor. Since we are only interested in the scaling behavior instead of the actual computation time required to solve the instances, we choose the number of sweeps as a difficulty measure over the walltime to reduce timing inconsistencies caused by various unrelated factors such as CPU idle time \cite{cpu}, inefficiency of the interpretive language \cite{interpretive}, and parallel efficiency \cite{parallel}. If one wishes to obtain an estimate of the scaling behavior of the number of arithmetic operations required for SA to find the ground state, one can simply rescale the number of sweeps by a factor of $n^2$ \cite{complexity_bible}, since the number of arithmetic operations required for a single SA iteration scales as $O(n^2)$ (see Appendix \ref{appendix:simu}). 

The SA solver we implement uses a linearly increasing $\beta$ schedule from 0.01 to $\log(n)$, such that the excited states are suppressed as $\frac{1}{n}$ \cite{sa_scale}. This schedule is scaled by a factor of $\frac{1}{\rho}$ in the high density regime to compensate for the increase in the magnitudes of weights (see Section \ref{appendix:simu}). To obtain a fair measure of hardness, it would be most ideal to use the optimal number of sweeps $N_{sweep}$ per run such that the total number of sweeps $N_{tot}$ is minimized for each generated instance. If $N_{sweep}$ is too small, then it is very unlikely for SA to discover the ground state in the highly non-convex energy landscape even if we were to perform many resets, and if $N_{sweep}$ is too large, then the rate of $\beta$ may be unnecessarily slow for the given difficulty, meaning that the descent in RBM energy is unnecessarily ``careful", making the solver take longer than needed to find the ground state. The task of finding the optimal $N_{sweep}$ is difficult, so to have a reasonable estimate of the optimal $N_{sweep}$ for hard instances, we first carefully tune $N_{sweep}$ for easy instances (of small size $n$ and small frustration $f$), and try to see how $N_{sweep}$ scales with $n$ and $f$. The optimal $N_{sweep}$ parameter can then be extrapolated for larger $n$ and $f$. More details of this method is given in Appendix \ref{appendix:sweep}. In the script \path{main.m} \cite{rudy_git}, $N_{sweep}$ is by default set to a value optimal for solving instances at the hardest loop density for every $n$ and $f$ (see Section \ref{Section_density} and Section \ref{Section_frus}).

\subsubsection{Simulated Annealing as Proxy}

We here make some brief comments justifying the use of single-spin SA for the hardness measurement of RBM instances, mainly by heuristically arguing that standard global algorithms are not expected to provide a significant improvement in performance, if any, over the single-spin SA on a bipartite constraint graph. We first note that since an RBM is highly connected, any attempt to perform a cluster update \cite{wolff, pt_iso} will likely be trivial, meaning either the clusters will be very small (at high temperature) or a cluster will span almost the entire RBM (at low temperature), which is effectively a small cluster up to the global $\mbb{Z}_2$ gauge. This is because the percolation threshold for a randomly weighted bipartite graph is sharp \cite{percolation}, and therefore cluster updates are not expected to provide any considerable improvement, if any, on the mixing rate of the simulation. As for algorithms based on the message passing techniques \cite{survey}, a spin update follows roughly an update rule in the orthant dynamics generated by spins distant 2 away (or next-nearest neighbors) on the graph. In addition, since the chromatic number is $2$ in a bipartite graph, a visible spin update will be informed by messages originated from all the other visible spins through the hidden layer. This means that in a complete bipartite graph, every spin will see approximately the same message, and an iteration of spin updates over the entire RBM will likely be trivial.

Finally, we point out that the hardness of the instance class is measured from the TTS (time to solution) statistics of a general incomplete stochastic solver (with the total number of SA sweeps being the proxy), meaning that the solver does not have to {\it prove} the optimum \cite{complete}, as it is able to check continuously against the planted solution and terminate immediately after the solution is reached. This TTS hardness measure is separate from the hardness of proving the optimum (required by a complete solver), and serves as a more accurate measure of the non-convexity of the energy landscape, whereas the hardness of proof generally scales with the clause density, which may not be directly related to the shape of the energy landscape. Most open-source incomplete solvers employs some implementation of WalkSAT as a subroutine \cite{minisat,glucose,open-wbo}, which is equivalent to the SA algorithm we use for this work. 

\subsection{Hardness Peak}
\label{Section_density}

From preliminary studies with a small sample of instances, we find that the locations of the hardness peaks for different system sizes are rather insensitive to the frustration index or the reset schedule, $N_{sweep}$. Therefore, we choose to perform this study with a frustration index of $f=0.05$ so that the instances are easy enough to be solved within a reasonable time window, meaning that we have the ability to solve a larger number of instances to reduce the uncertainty of the hardness measurement. As the first iteration of the optimization of SA (see Appendix \ref{appendix:sweep}), we optimize $N_{sweep}$ for solving instances generated at density $\rho = 0.47$, corresponding to the locations of the hardness peaks for small system sizes.

We attempt to locate the hardness peaks for instances of sizes ranging from $n=30$ to $n=200$ in increments of $10$. For each size $n$, we perform a hardness measurement for the following densities
\begin{equation*}
\rho = 0.1\times1.12^k \qquad k\in [[1,20]],
\end{equation*}
which is a geometric series from $0.1$ to around $1$, so the spacing of the densities will be uniform on a log scale. Measuring the hardness over these densities allows for a rough estimate of where the hardness peaks are located. For each $n$, we then ``zoom in" on the range of densities where the peaks are expected to be in, and measure the hardness over this range with a resolution of $0.005$. This finer resolution allows us to pinpoint more precisely the location of the hardness peaks.

For each tuple $\{n,\rho\}$, we generated 10,000 different instances and solve them with SA to estimate the sample distribution of $N_{tot}$, from which we extract the 95th percentile of $N_{tot}$ under the assumption that the distribution is approximately log-normal \cite{log_norm} (see Appendix \ref{appendix:sweep}). This is then reported as the measure of hardness at $\{n,\rho\}$.

\begin{figure}
\begin{center}
\includegraphics[scale=0.8]{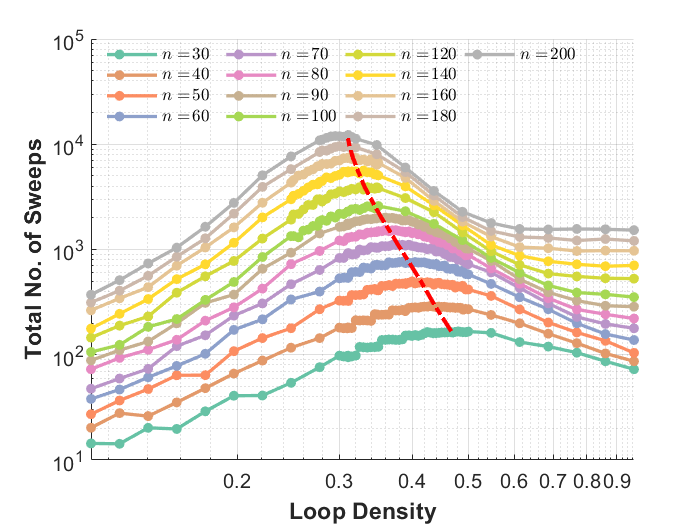}
\end{center}
\caption{Total number of sweeps, $N_{tot}$, versus loop density, $\rho$, plots for different values of system size, $n$. Note that the data points are more concentrated near the hardness peaks for higher resolution. Several plots for $n>100$ are omitted in the Figure for visual clarity, but they are nonetheless used to perform the fitting. The dashed red line shows the exponential fitting $\rho(n)$ for the relationship between the loop density of the hardness peak and $n$, Eq.~(\ref{rhofit}), which appears to plateau to a value around $0.3$ as the system size is increased.}
\label{test_density}
\end{figure}

The relationship between the hardness and loop density for various $n$ is shown in Figure \ref{test_density}. Note that for a larger system, the hardness peak is located at a smaller loop density, but the peak density appears to plateau to a constant value as the system size is increased. If we formally define the peak density, $\rho_{peak}(n)$, as the density at which the generated instances result in the highest 95th percentile of $N_{tot}$, we then find that the relationship between the peak density and the system size is well-fitted by the following decay function
\begin{equation}\label{rhofit}
\rho_{peak}(n) = 0.3035 + 0.2952\times \exp(-0.0196n).
\end{equation}

The appearance of the hardness peaks is expected and is common in most planted constraint satisfaction problems. In the regime of low loop density, the loops do not interact and the system can be factored into subsystems generated by individual loops. When there are too many loops, then the population of local minima in the energy landscape is very low (see Section \ref{limit}). Note that unlike the original frustrated-loop algorithm used by the quantum annealing community \cite{loop}, the loops in our algorithm do not intersect, so the decrease in difficulty is not attributed to the decrease in frustration due to destructive intersections (see Section \ref{int}). We thus attribute the decrease in difficulty solely to the random nature of the distribution of the negative weights, for which the LLN will almost always guarantee the absence of local minima in the regime of high loop density (see Appendix \ref{appendix:local_max}). Therefore, by using the structured-loop algorithm as described in Section \ref{modi_algo}, it is possible to retain the difficulty at high loop densities by enforcing certain structures on the loops. An empirical study of the hardness of structured loop instances in the high density regime will be presented in Section \ref{Section_struc}.

\subsection{Difficulty vs. Frustration}
\label{Section_frus}

\begin{figure}
\begin{center}
\includegraphics[scale=0.8]{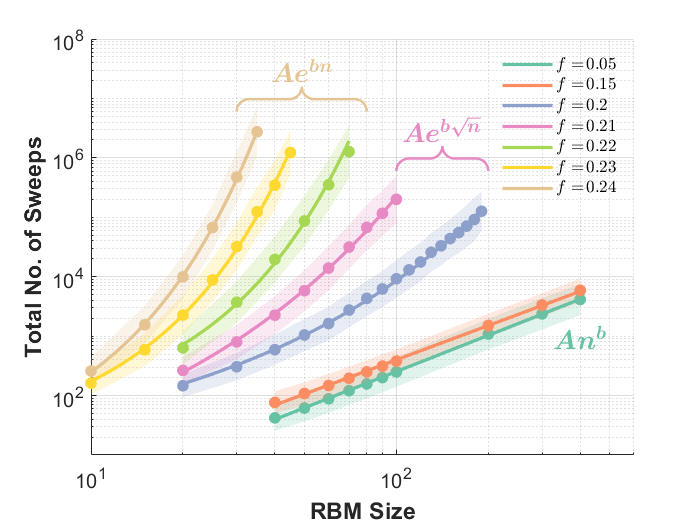}
\end{center}
\caption{Scaling behavior of the difficulty with respect to system size, $n$, for different frustration indexes, $f$. Note that the plots are on a log-log scale such that a polynomial scaling behavior will appear as a straight line, and an exponential scaling behavior will appear as a curve. The solid curves shown in the Figure are fitted curves for the estimated geometric mean of the $N_{tot}$ samples. The shaded area denotes the deviation from the mean by 0.5 times the estimated standard deviation in log space, which corresponds to roughly the 31th to 69th percentile estimates assuming an underlying log-normal distribution. For $f=\{0.05,0.1,0.15\}$, the data points are well fitted by a polynomial function of the form $Ab^n$ with parameters $\{A,b\} = \{0.0225,2.0259\},\,\{0.0284,2.0168\},\,\{0.0591,1.9150\}$ respectively. For $f=\{0.2,0.21\}$, the data are well fitted by an exponential function of the from $Ae^{b\sqrt{n}}$ with parameters $\{A,b\} = \{6.2343,0.7227\},\,\{1.0184,1.2283\}$. For $f=\{0.22,0.23,0.24\}$, they are well fitted by an exponential function of the form $Ae^{bn}$ with parameters $\{A,b\} = \{30.7289,0.1579\},\,\{13.3808,0.2564\},\,\{5.9172,0.3738\}$ }
\label{test_frus}
\end{figure}

After determining the peak density $\rho_{peak}(n)$ for each $n$, we now have the ability to generate the hardest instances for a given pair of $\{n,f\}$. This allows us to study the scaling behavior of the hardness over the hardest instances with respect to $n$ for different frustration indexes $f$. For this study, we use the following 8 frustration indexes
\begin{equation*}
f = \{0.05,0.10,0.15,0.2,0.21,0.22,0.23,0.24\}.
\end{equation*}
For each $f$, we choose an appropriate series of $n$ to estimate the scaling behavior. Since the higher the frustration, the more difficult the instances, only small values of $n$ can be used for highly frustrated instances to guarantee finding the solution within reasonable time.

For this study, we use a sweep schedule $N_{sweep}$ that is optimized for solving the hardest instances (see Appendix \ref{appendix:sweep}). The sample size of the instances for each pair of $\{n,f\}$ ranges from 100 to 10000 depending on how hard the instances are (the harder the instances, the smaller the sample size). As the 95th percentile estimate for $N_{tot}$ is noise dominated for small sample size, we instead opt to use geometric mean as a more reasonable measure of hardness.

The results are shown in Figure \ref{test_frus}, where the data points are fitted with either a polynomial or exponential function depending on the convergence of the fitting. The interesting result is that by tuning the frustration index, we can achieve different scaling laws for the difficulty. For low frustration indexes, or $f=\{0.05,0.10,0.15\}$, the scaling appears to be quadratic. For medium frustration indexes, or $f=\{0.2,0.21\}$, the scaling follows a sub-exponential trend of the form $Ae^{b\sqrt{n}}$. And for high frustration indexes, or $f=\{0.22,0.23,0.24\}$, the scaling follows the standard exponential growth of the form $Ae^{bn}$. 

The drastically different scaling laws that we can achieve by only slightly varying $f$ seems to hint at two discontinuous phase transitions in complexity driven by the frustration index, which may be potentially related to the thermal phase transition of the corresponding spin model. Note that phase transitions in classical \cite{phase} or quantum spin-glass models \cite{qphase} have been well-studied, and in the latter case, it is known that a {\it quantum} phase transition can be driven by the strength of the frustrated coupling terms \cite{qphase_frus}. However, we believe that such two-stage discontinuous phase transition has not been previously studied in any classical spin model. Therefore, this empirical result hints that the bipartite Ising spin glass may be an exceptional spin model rich in critical phenomena. 

\subsection{Structured-Loop Algorithm}
\label{Section_struc}

\begin{figure}
\begin{center}
\includegraphics[scale=0.8]{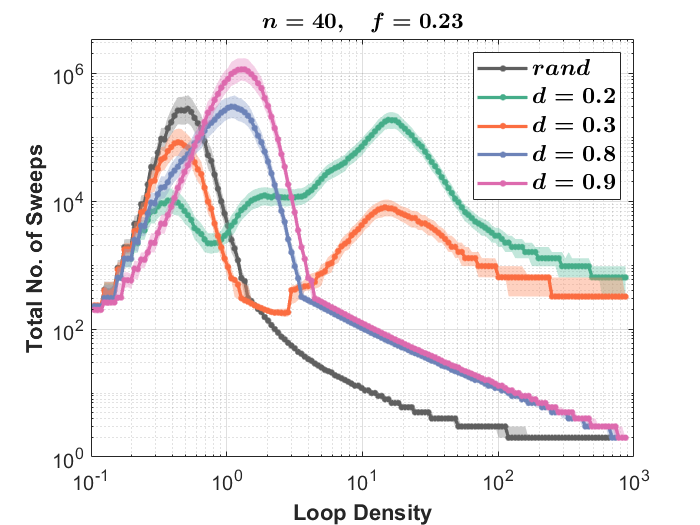}
\caption{\label{fig_struct} The hardness peaks of instances generated by the structured loop algorithm, measured as the median statistics, for a $40 \times 40$ RBM at frustration $f=0.23$, with the loop density varied from $0.01$ to around $10^3$. Various parameters of the negative block size, indicated in the legend as $d$, are used to generate different hardness variation behaviors, along with the regular hardness peak for the random loop algorithm shown as reference (indicated as 'rand'). Instances generated at $d=\{0.2,0.3\}$ display two hardness peaks, with the right peak being at an extensive density ($\sim 40$), while instances generated at $d=0.9$ display one prominent hardness peak at unitary density ($\sim 1$). In general, instances generated with the structured loop algorithm are harder in the high density regime relative to the random loop algorithm. }
\end{center}
\end{figure}

In Section \ref{modi_algo} and Appendix \ref{appendix:l_max}, it is argued that the structured loop algorithm is capable of generating hard instances at extensive loop density, and the hardness of the instances can be further increased by concentrating the negative weights into a small block of the weight matrix. This section mainly serves as an empirical support of this claim. The testings are done on an $40\times 40$ RBM at frustration $f=0.23$, and the loop density is varied over the following geometric series from 0.1 to around 1000,
\begin{equation*}
\rho = 0.1 \times 1.12^k \quad k\in [[1,80]].
\end{equation*}
For each $\{n,f,\rho\}$ triple, 10000 RBM instances are generated independently to ensure accurate percentile statistics. In this study, we use the median TTS as the measure of hardness, as higher percentile statistics (above the 95th percentile) appears to be extremely sensitive to the optimization parameters (see Appendix \ref{appendix:sweep}).

Note that a higher frustration index is chosen as it is required for the hardness to be maintained in the high density regime, in addition to display a more prominent easy-hard-easy-hard-easy transition. In Figure \ref{fig_struct}, the structured loop algorithm is compared against the random loop algorithm, and a drastically different hardness variation behavior for different frustration blocking scheme is observed, which suggests a possible complexity phase transition induced by the size of the negative block. 

In essence, the frustration index deviating from the maximal value of $f=0.25$ explicitly breaks the symmetry between the planted ground state and the planted metastable cluster with respect to the generation method (see Section \ref{sec_slog}), thus breaking also the symmetry of frustration blocking scheme between $d$ and $1-d$. The hardness variation behavior then changes critically at the threshold of $d=0.5$, at which point the hardness peak breaks into two, though an analytic description of why this occurs appears to be difficult. In addition, a larger deviation from the $d=0.5$ equal blocking scheme appears to result in harder instances, and an a heuristic argument is given in Appendix \ref{appendix:local_max} that attempts to explain this behavior in terms of local field {\it dispersion}.

\section{Connection to RBM Pre-training}
\label{app}

The focus of this work is mainly on generating weighted MAX-2-SAT instances of tunable difficulty with frustrated loops, and the RBM terminology was used so far mainly as a convenient denotation for a general bipartite spin-glass with Ising-type coupling. Going beyond terminology, there are indeed several substantial connections of this work to the field of machine learning, with some directly applicable to using RBM for unsupervised learning \cite{rbm}. These applications form two main branches, one practical and one heuristic. 

The practical application is that minimizing the RBM energy is equivalent to finding the mode of the RBM model distribution, which can be used to make highly informative weight updates during the pre-training \cite{mode_train}, which improves stability of the pre-training routine and enables the discovery of a model distribution with a much lower KL-divergence (against the data distribution) compared to standard methods such as contrastive divergence \cite{cd}. The heuristic connection is that the frustration index of the RBM contains much information on the behavior of the RBM during pre-training, meaning that it is possible to use the frustration index as an important indicator for certain properties of the RBM (see Section \ref{rbm_frus}) as a monitoring parameter for the RBM during training, in place of the KL-divergence whose evaluation is impractical for large system size. These connections are studied in more depth both analytically and empirically in another work \cite{mode_train}.

\subsection{Pre-training Using the Mode}

During pre-training, a gradient descent on the Kullback-Leibler (KL) divergence between the data and model distribution generates the following update rule for the RBM weights \cite{rbm}:
\begin{equation*}
\Delta W_{ij} = \mu(\braket{v_i h_j}_D - \braket{v_i h_j}_M),
\end{equation*}
where $\mu$ is the learning rate. The first term denotes the expected value of $v_i h_j$ over the {\it data distribution}, and the second term denotes the expected value over the {\it model distribution}. There are two main problems with this update rule. First, the KL divergence (as with any other loss function in machine learning) is highly non-convex \cite{sgd}, meaning that following the gradient will likely lead to a local minimum instead of the global one. Second, the model term (the second term in the update rule) is notoriously difficult to compute exactly since it requires summing over an exponential number of configurations \cite{rbm}:
\begin{equation*}
\braket{v_ih_j}_M = \sum_{\{ \mathbf{v},\mathbf{h} \}} p(\mathbf{v},\mathbf{h})v_ih_j,
\end{equation*}
where the expression can be simplified by tracing out the hidden layers, but still leaving an exponential number of visible layer configurations. In standard practice, this term is approximated using {\it contrastive divergence} \cite{cd}, which is a form of Markov chain Monte Carlo (MCMC). It is well known that CD is a highly unstable method and converges very poorly \cite{ad_mcmc}. 

The reason why CD performs poorly is because it is prone to being ``frozen" under one of the modes in a multi-modal distribution, and since it is energetically expensive to transit to another mode through single node flips, this essentially ``traps" the Markov chain and prevents it from effectively exploring the entire probability distribution \cite{mode_mcmc}. One obvious solution is to reinitialize the Markov chain at the global mode whenever it is trapped, and this has in fact been shown to be effective in improving the mixing time of the Markov chain \cite{dart,worm}. Recently, we have discovered that the usefulness of the mode goes beyond the effective re-initialization for the MCMC. In fact, using the mode {\it directly} to update the weights guarantees stability of the training routine and an effective exploration of the probability distribution, which results in the discovery of a model distribution with drastically lower KL-divergence compared to CD. We refer to such weight update method as {\it mode training}, which we explore extensively in our other work \cite{mode_train}.

\subsection{Frustration Index as an Indicator}
\label{rbm_frus}

\begin{figure}
\begin{center}
\includegraphics[scale=0.7]{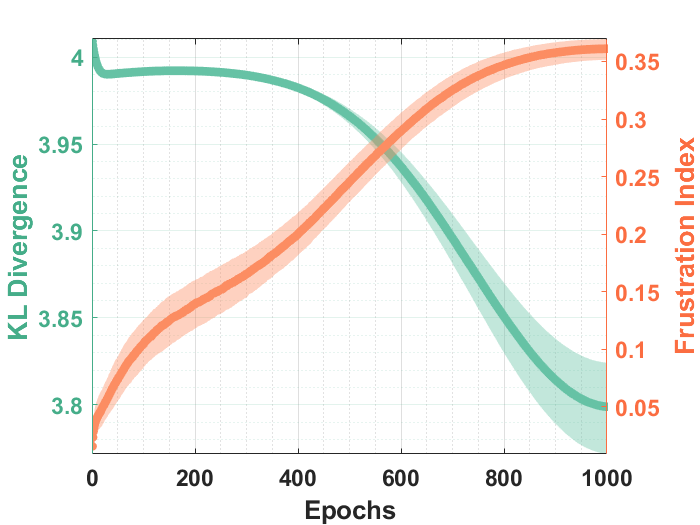}
\end{center}
\caption{The evolution of the KL-divergence and frustration index during the pre-training routine of 10000 randomly initialized RBMs, with the solid line representing the median statistics and the shaded area enclosing the 30-70 percentile. The RBM size is $9\times 12$, and the data set consists of 9 visible configurations of equal weights generated by a shifting bar of length 5 \cite{sb}. The pre-training algorithm employs the standard CD-1 estimation of the gradient \cite{cd}, with the learning rate decreasing linearly from $5 \times 10^{-2}$ to $5 \times 10^{-4}$ over 1000 epochs, and the full data set being used for every epoch. The computation of the KL divergence and frustration index is exact.}
\label{frus_evo}
\end{figure}

During RBM pre-training, it is ideal that the KL divergence \cite{kl} can be directly monitored, as it is an indicator of how ``close" the data and model distributions are, so we can observe directly the performance of the pre-training method. However, computing the KL divergence is usually impractical for large systems since it requires computing the partition function which involves an exponentially scaling number of sums, and current methods for estimating the KL divergence are mainly based on AIS (annealed importance sampling), which are rather inaccurate \cite{kl_est}.

Alternatively, the frustration index can be used as an indicator of the evolution of the model distribution. Note that during pre-training, the model distribution evolves to approximate the data distribution, which is generally a distribution with multiple dominant modes far apart from each other \cite{rbm}. This is equivalent to an RBM energy landscape with low energy states separated closely in energy but far in distance. In Section \ref{ffrus}, we argued that this is a characteristic of a highly frustrated system. Therefore, during pre-training, the frustration index is expected to rise to a relatively large value, and if it does not, then it may indicate that the model distribution is not converging to the data distribution. Computing the frustration index is computationally inexpensive once the minimum energy configuration is known \cite{frus}, so the evaluation of the frustration index can be scheduled to occur in conjunction with finding the global mode\footnote{Note that finding the ground state of the RBM exactly, though easier than computing the KL-divergence, is still an NP-hard problem. Practically speaking, we only have to find a dominant state with energy sufficiently close to the ground state, such that the frustration index can be effectively approximated. }.

In Figure \ref{frus_evo}, we plot the evolution of the frustration index as we pre-train a $9\times 6$ RBM using the standard CD-1 gradient estimate \cite{cd}, and we observe a clear increasing trend of the frustration index with the exception of a small dip in the middle of the pre-training routine. The frustration index is initially zero, and grows until it plateaus to a value around 0.18, which is at the critical value where difficulty scaling behavior transits from polynomial to exponential (see Fig.~\ref{test_frus}). A possible interpretation of this phenomenon is that the effectiveness of training the RBM with the mode is exchanged as the hardness of finding the mode of the PMF itself, which is expected from the ``no free lunch" theorem of optimization algorithms \cite{lunch}. Nevertheless, from a practical standpoint, if we were to use a highly effective MAX-SAT solver, obtaining a lower KL divergence using the {\it mode-training} method may be well worth the computational expense of finding the global mode.

In current practices, the goal of pre-training the RBM serves the purpose of acquiring better initial weight values for the supervised learning stage when the data labels are scarce \cite{why_pre}. On the other hand, the goal of supervised learning is to increase the certainty of the activation of out-nodes such that they correspond to the data labels \cite{deep}. However, it was recently discovered that pre-training of an RBM is also capable of achieving a higher certainty of the activation of hidden nodes \cite{hidacv}, and in our other work \cite{mode_train}, we explore deeper the correlation between the nodal activation in the hidden layer and the frustration index of the RBM. A mediator of this correlation may come in the form of information compression, in the sense that the increase of the activation of hidden nodes is potentially related to the drift-diffusion transition on the information bottleneck (IB) curve \cite{rbm_info}, which itself may parallel the phase transition phenomenon driven by the frustration index (see Section \ref{Section_frus}) as the RBM is pre-trained. It will be interesting to explore these potentially useful directions of study even further to establish the connection between the frustration index and the dynamical properties of the RBM during unsupervised and supervised training.

\section{Code Availability}

A MATLAB implementation of the random and structured loop algorithm is available in the Github repository \path{PeaBrane/Ising-Simulation}, which also includes the code for the SA solver used to perform the hardness measurement. In addition to the codes used for this work, the repository also contains codes for generating hard instances on a 3D cubic lattice \cite{loop,tiling}, and standard algorithms for the simulation of the lattice \cite{pt_iso}. Scripts for converting the instances directly into a (w)cnf file is also available in the repository. 

\section{Conclusion}

In this paper, we reformulated the frustrated-loop algorithm, originally conceived to test quantum annealers, in a way such that it can be directly applied to generating weighted MAX-2-SAT instances of tunable hardness. In addition, we introduced the structured-loop algorithm for the purpose of extending the hardness into the high density regime. An unexpected discovery is the possible two-stage phase transitions in the hardness scaling behavior driven by the frustration index, which will hopefully motivate further theoretical work on the characterization of spin-glass models via the frustration index. Since both algorithms are capable of generating instances of tunable hardness over a wide range of clause densities, they can be used in conjunction to effectively evaluate the performance of a wide class of solvers \cite{mse}. It would be also interesting to test unconventional solvers that operate under continuous-time dynamics \cite{analog_maxsat}, and solvers based on the memcomputing architecture \cite{DMM2,max-sat} that employs memory-assisted dynamics to induce long-range correlations of spins \cite{sheldon2019taming}. We leave these testings for future work.

\section{Acknowledgments}
Y.R.P. and M.D. are supported by DARPA under grant No. HR00111990069. H.M. acknowledges support from a DoD-SMART fellowship. M.D. acknowledges partial support from the Center for Memory and Recording Research at UCSD. 

\raggedbottom
\pagebreak

\appendix

\section{Penalty Function}
\label{appendix:rbi}

In this section, we show the existence of a penalty function $C(\mbf{v}, \mbf{h})$ in the conversion from a general QUBO instance to a bipartite QUBO instance such that the optimum is left invariant. We formalize this statement as follows.

\begin{proposition}
Given a QUBO instance 
\begin{equation*}
\sum_{i=1}^n B_i x_i + \sum_{i=1}^n \sum_{j=i+1}^n Q_{ij}x_ix_j
\end{equation*}
with (any one of) the maximum (maxima) being $\mbf{x'}$. If we let 
\begin{equation*}
c = \sum_{i=1}^n |B_i| + \sum_{i=1}^n \sum_{j=i+1}^m |Q_{ij}|,
\end{equation*}
and the penalty function be
\begin{equation*}
C(\mbf{v}, \mbf{h}) = 2c \sum_{i=1}^n ( v_i + h_i - v_ih_i ),
\end{equation*}
then (one of) the maximum (maxima), $\{\mbf{v'}, \mbf{h'}\}$, of the bipartite QUBO problem,
\begin{equation*}
E(\mbf{v}, \mbf{h}) = E_0(\mbf{v}, \mbf{h}) + C(\mbf{v}, \mbf{h}),
\end{equation*}
must satisfy $\mbf{v'} = \mbf{h'} = \mbf{x'}$, where $E_0(\mbf{v}, \mbf{h})$ is given as
\begin{equation*}
E_0(\mbf{v}, \mbf{h}) = \sum_{i=1}^n B_iv_i + \sum_{i=1}^n \sum_{j=i+1}^n Q_{ij}v_ih_j.
\end{equation*}
\end{proposition}

\begin{proof}
WLOG, we can assume that the original QUBO problem is non-degenerate so there is only one maximum. We first begin by noting that 
\begin{equation*}
v_i + h_i - v_i h_i = 
\begin{cases}
0 \quad &\text{if} \; v_i = h_i, \\
-1 &\text{if} \; v_i \neq h_i.
\end{cases}
\end{equation*}
From this, we see that the maximum of the bipartite QUBO problem must satisfy $\mbf{v'} = \mbf{h'}$. 

If this were not the case, then we let $\mbf{v'} = \mbf{h'}$ be the optimum with the condition $\mbf{v} = \mbf{h}$ explicitly enforced, and $\mbf{v''} = \mbf{h''}$ be the optimum without enforcing any condition, so $E(\mbf{v''}, \mbf{h''}) > E(\mbf{v'}, \mbf{h'})$. We denote $I = \{ i \cond v''_i \neq h''_i\}$ as the set of indices where $\mbf{v''}$ and $\mbf{h''}$ differ. WLOG, we let $I = [[1,n'']]$, where $n'' < n$. Then we have
\begin{equation*}
\begin{split}
E(\mbf{v''}, \mbf{h''}) - E(\mbf{v'}, \mbf{h'})
\leq
& \, |E_0(\mbf{v'}, \mbf{h'})| + |E_0(\mbf{v''}, \mbf{h''})| - 2n' c \\
\leq
& \, c + c - 2n'c = 2(1-n') c \\
< & \, 0,
\end{split}
\end{equation*}
where the second inequality is due to the fact that for $\forall \{ \mbf{v}, \mbf{h} \}$, $|E_0(\mbf{v}, \mbf{h})| \leq c$ by the triangle inequality. This contradicts with the assumption that $E(\mbf{v''}, \mbf{h''}) > E(\mbf{v'}, \mbf{h'})$, so the maximum of $E(\mbf{v}, \mbf{h})$ must satisfy $\mbf{v} = \mbf{h}$.

This implies that
\begin{equation*}
\{ \mbf{v'}, \mbf{h'} \} = \argmax_{\{ \mbf{v}, \mbf{h} \}} E(\mbf{v}, \mbf{h}) = \argmax_{\{ \mbf{v}, \mbf{h} \cond \mbf{v} = \mbf{h} \}} E(\mbf{v}, \mbf{h}),
\end{equation*}
so we can find the maximum of $E(\mbf{v}, \mbf{h})$ by first setting $\mbf{v} = \mbf{h}$, which gives us
\begin{equation*}
E(\mbf{v}, \mbf{v}) = \sum_{i=1}^n B_iv_i + \sum_{i=1}^n \sum_{j=i+1}^m Q_{ij}v_iv_j - n,
\end{equation*}
which is equivalent to the expression for the original QUBO energy, so we have $\mbf{v'} = \mbf{h'} = \mbf{x'}$.
\end{proof}

\section{Switching Subset and Distance}
\label{appendix:metric_proof}

Given two states, $\mbf{s} = ( \mbf{v}, \mbf{h} )$ and $\mbf{s'} = ( \mbf{v'}, \mbf{h'} )$, it is convenient to denote the indices of visible spins that are different between the two states as
\begin{equation*}
I(\mbf{s}, \mbf{s'}) = \{ i \bcond v_i \neq v'_i \},
\end{equation*}
and the indices of differing hidden spins as
\begin{equation*}
J(\mbf{s}, \mbf{s'}) = \{ j \bcond h_j \neq h'_j \}.
\end{equation*}
Furthermore, we denote the cardinality of the two sets as, $n' = |I(\mbf{s}, \mbf{s'})|$ and $m' = |J(\mbf{s}, \mbf{s'})|$, which represent the numbers of differing visible and hidden spins respectively. \\

Under this denotation, we can write the switching subset as
\begin{equation*}
F(\mbf{s}, \mbf{s'}) = (I \times J^c) \,\cup\, (I^c \times J),
\end{equation*}
where $I^c = [[1,n]]/I$ and $J^c = [[1,m]]/J$. It is then obvious that the cardinality of the switching subset is given as
\begin{equation*}
\begin{split}
|F(\mbf{s}, \mbf{s'})| &= n'\times (m-m') + (n-n')\times m' \\
&= n'm + nm' - 2n'm',
\end{split}
\end{equation*}
so the distance is given as
\begin{equation*}
d(\mbf{s}, \mbf{s'}) = \frac{|F(\mbf{s},\mbf{s'})|}{nm} = \frac{n'}{n} + \frac{m'}{m} - 2\frac{n'm'}{nm}.
\end{equation*}
We now show that the space defined by this distance is a pseudometric space, and states in this space are distinguishable up to a global spin flip.

\begin{proposition}
$d(\mbf{s}, \mbf{s'})$ is a pseudometric, with $d(\mbf{s}, \mbf{s'}) = 0$ if and only if $\{n',m'\} = \{0,0\}$ or $\{n',m'\} = \{n,m\}$.
\end{proposition}

\begin{proof}
Since the distance function $d(\mbf{s}, \mbf{s'})$ is just $|F(\mbf{s}, \mbf{s'})|$ divided by some constant factor, it is sufficient to prove the proposition for $|F(\mbf{s}, \mbf{s'})|$. We first show the second part of the proposition. Note that the function,
\begin{equation*}
|F(\mbf{s}, \mbf{s'})| = n'm + nm' - 2n'm',
\end{equation*}
evaluates to $0$ if $\{n',m'\} = \{0,0\}$ or $\{n',m'\} = \{n,m\}$. To show the converse, we note that the equation
\begin{equation*}
\begin{split}
& n'm + nm' - 2n'm'  \\
= & n'(m-m') + (n-n')m' = 0,
\end{split}
\end{equation*}
under the conditions $n > 0$, $n \geq n' \geq 0$, $m > 0$, and $m \geq m' \geq 0$, has solutions $\{n',m'\} = \{0,0\}$ and $\{n',m'\} = \{n,m\}$. 

To show that $|F(\mbf{s}, \mbf{s'})|$ is a pseudometric, we have to show that it is non-negative, symmetric, and satisfies the triangle inequality. First, we note that $|F(\mbf{s}, \mbf{s'})|$ is trivially non-negative as it is a cardinal number. Second, we have $|F(\mbf{s}, \mbf{s'})| = |F(\mbf{s'}, \mbf{s})|$, as the same spins are flipped to make the forward and reverse transitions, $\mbf{s} \rightarrow \mbf{s'}$ and $\mbf{s'} \rightarrow \mbf{s}$. 

Finally, we have to show that given any three states, $\{ \mbf{s}, \mbf{s'}, \mbf{s''} \}$, the inequality $|F(\mbf{s}, \mbf{s''})| \leq |F(\mbf{s},\mbf{s'})|  + |F(\mbf{s'},\mbf{s''})|$ is true. We first note that
\begin{equation*}
F(\mbf{s}, \mbf{s''}) = F(\mbf{s}, \mbf{s'}) \,\triangle\, F(\mbf{s'}, \mbf{s''}),
\end{equation*}
where $\triangle$ denotes the symmetric difference operation. We then have
\begin{equation*}
F(\mbf{s}, \mbf{s''}) \subseteq F(\mbf{s}, \mbf{s'}) \,\cup\, F(\mbf{s'}, \mbf{s''}),
\end{equation*}
which directly implies the inequality stated above. Therefore, $|F(\mbf{s},\mbf{s'})|$ is a pseudometric.
\end{proof}

\section{Energy Gaps in a Random RBM}
\label{appendix:variance}

\begin{proposition}
Given an $n \times m$ RBM with iid weights, $\mbf{W}$, normally distributed with mean $\mu$ and variance $\sigma^2$. For two random states $\{ \mbf{s}, \mbf{s'} \}$ of distance $d_0$ apart, the expected energy gap between the two states is given as
\begin{equation*}
\E_{\{ \mbf{W},\mbf{s},\mbf{s'} \}}\Big( E(\mbf{s'}) - E(\mbf{s}) \bcond d(\mbf{s}, \mbf{s'}) = d_0 \Big)
= 0,
\end{equation*}
and the variance is given as
\begin{equation*}
\Var_{\{ \mbf{W},\mbf{s},\mbf{s'} \}}\Big( E(\mbf{s'}) - E(\mbf{s}) \bcond d(\mbf{s}, \mbf{s'}) = d_0 \Big)
= 4nmd_0 (\mu^2 + \sigma^2).
\end{equation*}
\end{proposition}

\begin{proof}
Let $F(\mbf{s},\mbf{s'})$ be the switching subset from $\mbf{s}$ to $\mbf{s'}$, then from Eq. (\ref{eng_diff}), we have
\begin{equation*}
E(\mbf{s'}) - E(\mbf{s}) = 2 \sum_{F(\mbf{s}, \mbf{s'})} W_{ij}v_i h_j.
\end{equation*}
Note that the distance condition is equivalent to $|F| = nmd_0$, then the expected energy gap is given as
\begin{equation*}
\begin{split}
& \E_{\{ \mbf{W},\mbf{s},F \}} \Big( 2 \sum_F W_{ij}v_ih_j \bcond |F| = nmd_0 \Big) \\
= & 2 \sum_F \E( W_{ij} v_i h_j).
\end{split}
\end{equation*}
However, $\E( W_{ij} v_i h_j) = \E( W_{ij}) \E(v_i) \E(h_j) = 0$, so the expected energy gap is zero.

The conditional variance is given as
\begin{equation*}
\begin{split}
\Var\Big( 2 \sum_F W_{ij}v_ih_j \Big)
& = \E\Big( \big( 2 \sum_F W_{ij}v_ih_j \big)^2 \Big) \\
& = 4 \E_{F} \Big( \sum_{(i,j) \in F} \sum_{(i',j') \in F} \E_{ \{ \mbf{W},\mbf{s} \}} \big( W_{ij}W_{i'j'}v_iv_{i'}h_jh_{j'} \big) \bcond |F|=nmd_0 \Big) \\
& = 4 \E_{F} \Big( \sum_F W_{ij}^2 \bcond |F|=nmd_0 \Big) \\
& = 4nmd_0 (\mu^2 + \sigma^2),
\end{split}
\end{equation*}
noting that $\E(v_iv_{i'}h_jh_{j'}) = \delta_{ii'}\delta_{jj'}$, which evaluates to 1 only when $i=i'$ and $j=j'$, and 0 otherwise.
\end{proof}

\section{Maximum Frustration of a $2\times m$ RBM}
\label{appendix:2m}

\begin{proposition}
The upper bound of the frustration index of a $2 \times m$ RBM is 0.25.
\end{proposition}

\begin{proof}
We first consider a $2\times m$ RBM weight matrix that is gauged such that the ground state is $\mbf{+1}$, then clearly the sum of each column has to be non-negative, so we cannot have two negative elements in the same column. This means that we can permute the rows and columns such that all the negative weight elements are concentrated on the upper left and lower right corner. We can then divide the matrix into six blocks, 
\begin{equation*}
\begin{split}
B_1=1\times [[1,m_1]]\quad B_2=1\times [[m_1+1,m_2]]\quad B_3=1\times [[m_2+1,m]] \\
B_4=2\times [[1,m_1]]\quad B_5=2\times [[m_1+1,m_2]]\quad B_6=2\times [[m_2+1,m]],
\end{split}
\end{equation*}
such that the elements of blocks $B_1$ and $B_6$ are all negative, and the elements of the rest of the blocks are all non-negative. For clarity, we denote $S(B_i) = \sum_{w\in B_i} |w_i|$ as the sum of the absolute values of all elements in the $i$-th block. It is then clear that
\begin{equation*}
S(B_4) \geq S(B_1) 
\qquad
S(B_3) \geq S(B_6),
\end{equation*}
which follows directly from the application of the positive sum condition to each column. Furthermore, if we apply the positive sum condition on set $B_1 \cup B_2 \cup B_6$ and $B_1 \cup B_5 \cup B_6$, then we get
\begin{equation*}
S(B_2) \geq S(B_1) + S(B_6)
\qquad
S(B_5) \geq S(B_1) + S(B_6).
\end{equation*}

Combining these relations between the blocks, we can derive the following
\begin{equation*}
\begin{split}
\sum_i S(B_i)
& = \big( S(B_1) + S(B_4) \big) + \big( S(B_2) + S(B_5) \big) + \big( S(B_3) + S(B_6) \big) \\
& \geq 2 S(B_1) + 2 \big( S(B_1) + S(B_6) \big) + 2 S(B_6) \\
& = 4 \big( S(B_1) + S(B_6) \big),
\end{split}
\end{equation*}
so we see that the sum of the absolute values of all elements is at least four times the sum of the absolute values all negative elements. Therefore, from Eq.~(\ref{frus_eq}), we see that the frustration index must be smaller than $0.25$. 
\end{proof}

\section{Intersection Event}
\label{appendix:poisson}

If we are randomly dropping loop atoms on a $K_{n,m}$ bipartite graph, then the probability that the loop overlaps with any given edge is given by
\begin{equation*}
p = \frac{4}{nm}.
\end{equation*}
Therefore, if we were to drop one loop atom, then the probability that the edge receives a positive contribution is $\frac{3}{4}p$; the probability that it receives a negative contribution is $\frac{1}{4}p$; and the probability that it receives no contribution is $1-p$. If we denote the total number of random loop atoms as $N$, then the probability that any given edge receives $k_1$ negative contributions and $k_2$ positive contributions is given as
\begin{equation*}
P(k_1,k_2) = {N \choose {k_1,k_2}}\Big( \frac{1}{4}p \Big)^{k_1} \Big( \frac{3}{4}p \Big)^{k_2}(1-p)^{N-k_1-k_2}.
\end{equation*}
The expected number of intersections is then simply $\E(\min\{k_1,k_2\})$. 

To obtain an analytic expression for this expected value, we have to make a few simplifications. First, we can assume that $p\ll 1$, which is justified if the graph is large. If we denote $\lambda = Np$, then the marginal distributions of $k_1$ and $k_2$ are approximately Poisson distributions
\begin{equation*}
P(k_1) = \frac{e^{-\lambda/4}}{k_1!} \Big( \frac{\lambda}{4} \Big)^{k_1}
\qquad 
P(k_2) = \frac{e^{-3\lambda/4}}{k_2!} \Big( \frac{3\lambda}{4} \Big)^{k_2}.
\end{equation*}
We can also assume that $k_1$ and $k_2$ are approximately independent, which gives us
\begin{equation}
\label{mink}
\mathbf{E}(\min\{k_1,k_2\})=\frac{1}{2}\lambda-\frac{1}{2}e^{-\lambda}\sum_{k=0}^{\infty}k\Big( 3^{-k/2} + 3^{k/2} \Big) \text{I}_k \big( \frac{\sqrt{3}}{2}\lambda \big),
\end{equation}
where $\text{I}_k$ is the modified Bessel function of the first kind. As $\lambda$ increases, the pdf of $\min\{k_1,k_2\}$ approaches the pdf of $k_1$, and the expected value approaches $\lambda/4$ from below, which makes sense because the relative spacing of the random variables increases, and we effectively have $\min\{k_1,k_2\}\approx k_1$. In other words, we have
\begin{equation*}
\E( \min\{k_1,k_2\} ) \approx \frac{\lambda}{4} = \frac{N}{4nm}.
\end{equation*}

\section{Generating a $2\times 3$ gauged RBM}
\label{appendix:gen}

The negation of statement~(\ref{f2}) is given as follows:
\begin{equation}
\label{far_neg}
\begin{split}
&\text{If $\mathbf{y}$ satisfies } \mathbf{L}^T\mathbf{y}\geq 0 \text{ and } \mathbf{y}\geq 0 \text{ ,}\\
&\text{then it must also satisfy } \mathbf{w}^T\mathbf{y}\geq 0.
\end{split}
\end{equation}
The goal is to prove this statement true for a $2\times 3$ gauged RBM weight matrix $\mathbf{w}$. We first note that there are $12$ possible loop atoms for the system: 4 for the leftmost $2\times 2$ block, 4 for the rightmost $2\times 2$ block, and 4 for the union of the leftmost and rightmost column. Given any one of the leftmost loop, $\mathbf{l}$, the inequality $\mathbf{l}^T\mathbf{y}$ implies
\begin{equation*}
l_{11}y_{11} + l_{21}y_{21} + l_{12}y_{12} + l_{22}y_{22} \geq 0.
\end{equation*}
Each of the four leftmost loops corresponds to assigning one of the four edges, $\{l_{11},l_{12},l_{21},l_{22}\}$, to negative, and this results in four inequalities
\begin{equation*}
\begin{split}
y_{11} + y_{21} + y_{12} \geq y_{22}, \\
y_{21} + y_{12} + y_{22} \geq y_{11}, \\
y_{12} + y_{22} + y_{11} \geq y_{21}, \\
y_{22} + y_{11} + y_{21} \geq y_{12}.
\end{split}
\end{equation*}
WLOG, we assume that $y_{22}$ is the maximum of the four $y$ values, then the four inequalities reduce to the following inequality
\begin{equation*}
y_{11} + y_{21} + y_{12} \geq y_{22},
\end{equation*}
noting that the $y$ values are non-negative. A similar argument applies to the remaining 8 loop atoms.

For the weight matrix, $\mathbf{w}$, WLOG we can assume that the negative elements are $w_{11}$ and $w_{12}$, then the positive sum condition implies that
\begin{equation*}
\begin{split}
&w_{11}+w_{21}\geq 0, \quad w_{12}+w_{22}\geq 0, \\
&w_{11}+w_{12}+w_{13}\geq 0, \quad w_{11}+w_{12}+w_{23}\geq 0.
\end{split}
\end{equation*}
We can then derive the following relationship
\begin{equation*}
\begin{split}
&\mathbf{w}^T\mathbf{y} \\
= \, & w_{11}y_{11} + w_{12}y_{12} + w_{13}y_{13} + w_{21}y_{21} + w_{22}y_{22} + w_{23}y_{23} \\
\geq \, & y_{11}w_{11} + y_{12}w_{12} + y_{13}(-w_{11}-w_{12}) \\
+ \, & y_{21}(-w_{11}) + y_{22}(-w_{12}) + y_{23}(-w_{11}-w_{12}) \\
\geq \, & -w_{11}(y_{13}+y_{21}+y_{23}-y_{11}) \\
- \, & w_{12}(y_{13}+y_{22}+y_{23}-y_{12}) \\
\geq \, & 0,
\end{split}
\end{equation*}
since $w_{11}\leq 0$ and $w_{12}\leq 0$ by construction, and $y_{13}+y_{21}+y_{23}\geq y_{11}$ and $y_{13}+y_{22}+y_{23}\geq y_{12}$. Therefore, statement~(\ref{far_neg}) is true, which implies that any gauged $2\times 3$ RBM (which necessarily has $f \leq 0.25$) can be generated with loop atoms. 

\section{Local Minima}
\label{appendix:local_max}

Consider $n$ iid random variables, $\{x_1,x_2,...,x_n\}$, with the following PMF
\begin{equation*}
P(x_i=-\alpha)=\frac{1}{4} \qquad P(x_i=1)=\frac{3}{4},
\end{equation*}
where $\alpha \in (0,1]$. It is clear that the sum of the elements in this set is positive w.h.p. (with high probability) in the limit of large $n$. However, if we randomly select $n'$ elements and negate their signs, then it can be shown that the probability that the sum of the elements is positive is given by
\begin{equation*}
\frac{1}{2} \Big( 1+ \erf \Big( k \frac{n-2n'}{\sqrt{6n}} \Big) \Big)
\end{equation*}
in the limit of large $n$, where $k$ is related to $\alpha$ as
\begin{equation*}
k = \frac{3-\alpha}{\alpha+1}.
\end{equation*}

Now, for the sake of simplicity, consider a $n \times n$ random RBM with weights $\mbf{W}$ whose elements are assigned randomly as $-\alpha$ and $+1$ with probabilities $1/4$ and $3/4$, respectively (corresponding to the regime of high loop density, $\rho = O(n)$). Clearly, in the limit of large $n$, the RBM satisfies the positive-sum condition, meaning that its ground state is $\mbf{+1}$. Given $0 \leq n_1,n_2 \leq n$, we consider the switching subset 
\begin{equation*}
F = \big( [1,n_1]\times [n_2+1,n] \big) \,\cup\, \big( [n_1+1,n]\times [1,n_2] \big),
\end{equation*}
where integers are assumed. If the state related to the ground state by this switching subset is a local minimum, then clearly the following conditions have to be satisfied
\begin{equation*}
\begin{split}
\forall i\in [1,n_1], \qquad &\sum_{j=1}^{n_2} W_{ij} -  \sum_{j=n_2+1}^{n} W_{ij} \geq 0 \\
\forall i\in [n_1+1,n], \qquad &\sum_{j=n_2+1}^{n} W_{ij} -  \sum_{j=1}^{n_2} W_{ij} \geq 0  \\
\forall j\in [1,n_2], \qquad &\sum_{i=1}^{n_1} W_{ij} -  \sum_{i=n_1+1}^{n} W_{ij} \geq 0 \\
\forall j\in [n_2+1,n], \qquad &\sum_{i=n_1+1}^{n} W_{ij} -  \sum_{i=1}^{n_1} W_{ij} \geq 0. \\
\end{split}
\end{equation*}
If we make the approximation that the partial sum over a row and the partial sum over a column are independent (which is justified because the correlation between the two sums is only due to one single element at the intersection), then the probability that all the above conditions are satisfied is
\begin{equation*}
\begin{split}
p(n_1,n_2) = \frac{1}{2^{2n}}
&\erfc \Big( k\frac{n-2n_2}{\sqrt{6n}} \Big)^{n_1} 
\erfc\Big( k\frac{2n_2-n}{\sqrt{6n}} \Big)^{n-n_1} 
\erfc\Big( k\frac{n-2n_1}{\sqrt{6n}} \Big)^{n_2} 
\erfc\Big( k\frac{2n_1-n}{\sqrt{6n}} \Big)^{n-n_2} . \\
\end{split}
\end{equation*}

Note that there are ${n\choose n_1}$ ways to flip $n_1$ spins in the visible layer and ${n\choose n_2}$ ways to flip $n_2$ spins in the hidden layer. If we further choose to ignore the potential correlations between the local minima, then the expected number of local minima of an $n\times n$ RBM is given by
\begin{equation*}
\sum_{n_1=0}^n\sum_{n_2=0}^n 
p(n_1,n_2){n\choose n_1}{n\choose n_2} 
-p(0,0),
\end{equation*}
where the reason to subtract $p(0,0)$ is to discount the planted ground state being a trivial local minimum. 

It can be shown that this value scales poorly with $n$ and $k$ (which is inversely related to $\alpha$). In other words, for instances generated with the random loop algorithm at high density, it will be difficult for local minima to populate the energy landscape for large system sizes or small frustration indices.

\section{Planted Metastable Cluster}
\label{appendix:l_max}

Again, we consider the switching subset $F=B_1\cup B_4$ as given in Section \ref{degen}. The goal is to show that the state related to the ground state by $F$ satisfies the inequality conditions of a local minimum by a relatively large margin (meaning that any single spin-flip is expected to incur a large increase in energy), so that any states sufficiently ``close" to this state are likely also local minima. 

Since the structured loop algorithm is invariant under a matrix transpose and an exchange of the number of left and upper loops, we can simply focus on the sum of the elements in each individual row:
\begin{equation*}
\label{cond_gen}
\begin{split}
\forall i\in I \qquad &\sum_{J^c} W_{ij} \geq \sum_J W_{ij} \\
\forall i\in I^c \qquad &\sum_J W_{ij} \geq \sum_{J^c} W_{ij},
\end{split}
\end{equation*}
where we have $B_1 = I \times J$ and $B_4 = I^c \times J^c$. 

$\forall i\in I$, we can divide the row of the matrix into two halves, one half in $B_1$ and one half in $B_2$. We let the sums of the elements in the two halves be $s_1$ and $s_2$, respectively. Furthermore, let $n_1$, $n_2$, and $n_3$ be the numbers of left loops, upper loops, and center loops with vertices in row $i$. Each left loop contributes $1-\alpha$ units of weight to $s_1$; each upper loop contributes either $-\alpha$ or $+1$ units of weight to $s_1$ and $+1$ unit of weight to $s_2$; each center loop contributes $-\alpha$ units of weight to $s_1$ and $+1$ unit of weight to $s_2$. It is then clear that the difference between the two sums is given as follows
\begin{equation*}
s_2 - s_1 \in \big[ \, n_3(1+\alpha) - n_1(1-\alpha) \, , \, (n_2+n_3)(1+\alpha) - n_1(1-\alpha) \, \big],
\end{equation*}
whose lower bound should be non-negative in order to enforce the local minimum condition on row $i$. This gives us
\begin{equation*}
n_3(1+\alpha) - n_1(1-\alpha) \geq 0 \implies \frac{n_3}{n_1} \geq \frac{1-\alpha}{1+\alpha},
\end{equation*}
which is true for every $\alpha\in [0,1]$ if $n_3 \geq n_1$. Similarly, $\forall i\in I^c$, we can again divide the matrix row into two halves, one half in $B_3$ and one half in $B_4$, and we denote the sums over the two halves as $s_3$ and $s_4$. An upper loop does not contribute to either sum; a left loop contributes $2$ units of weight to $s_3$; and a center loop contributes $+1$ unit of weight to $s_3$ and $s_4$ each. Then the difference between the two sums is
\begin{equation*}
s_3 - s_4 = 2 n_1 \geq 0,
\end{equation*}
which is always true. Therefore, $n_3 \geq n_1$ on each row guarantees the local minimum condition, and similarly, $n_3 \geq n_2$ on each column guarantees the local minimum condition as well. These conditions will likely be satisfied for every row and column if we choose the number of center loops to be sufficiently large.

\subsection{Concentration of Frustration}

We see that a center loop is more conducive to the population of local minima than the other two types of loops, so from here on, we focus exclusively on instances generated by center loops alone. For instances generated at frustration index $f=0.25$, the magnitudes of the negative and positive loop edges are both $1$, so the structured loop algorithm is symmetric with respect to the exchange of $B_1$ and $B_4$. This means that parameterizing the size of the negative block with $d$ is equivalent to $1-d$. This symmetry is broken if we choose the magnitude of the negative edge weight to be slightly below $1$, or $\alpha = 1-\epsilon$ (where $0<\epsilon\ll 1$), and a smaller value of $d$ is generally favored for generating hard instances at high loop density. Intuitively, having a smaller value of $d$ ``concentrates" the negative weights into a smaller block, meaning that the expected value of the negative weights will be large, thus giving more strength to the weight structure for ``misguiding" local solvers away from the planted solution. We now make this argument slightly more formal.

For ease of analysis, we focus on the marginal distribution of the local field at $h_1$, which is given as
\begin{equation*}
\sum_i W_{i1}v_i.
\end{equation*}
This distribution is conditioned on having $N$ loops intersect with the first column of the weight matrix, and having an $r$ fraction of visible spins aligning with the planted ground state, which, WLOG, we assume to be $\mbf{+1}$. If we further denote
\begin{equation*}
n_1 = B(nd,r)
\end{equation*}
as the number of visible spins aligning with the planted solution in the upper-left block (where $B$ denotes the binomial distribution), then the local field can be expressed as a random variable parameterized as
\begin{equation*}
\begin{split}
L =\, & 2 B\Big( N,\frac{n_1}{nd} \Big)(\epsilon-1) + 2 B\Big(N, \frac{nr-n_1}{n(1-d)}\Big) - \epsilon N \\
=\, & 2 B\Big( N,\frac{B(nr,d)}{nd} \Big)(\epsilon-1) + 2 B\Big(N, \frac{nr-B(nr,d)}{n(1-d)}\Big) - \epsilon N.
\end{split}
\end{equation*}
Using the linearity of the expected value operator, one can easily show that the expected local field is $N\epsilon(2r-1)$, which is independent of the parameter $d$, and is proportional to the fraction of hidden spins aligned with the planted ground state. 

The parameter $d$ is relevant when we evaluate the total variance of the local field (over the probability measure of $n_1$), which can be computed via the law of total variance
\begin{equation*}
\Var(L) = \Var\big( \E(L\cond n_1) \big) + \E\big( \Var(L\cond n_1) \big).
\end{equation*}
It is convenient for us to define the {\it dispersion} of the local field as its relative standard deviation,
\begin{equation*}
c_v(L) = \frac{\sqrt{\Var(L)}}{\E(L)},
\end{equation*}
which provides a measure of the uncertainty of the local field for each spin normalized against the weight magnitude. In Figure \ref{disperse}, we see that the dispersion is greater if we have a greater concentration of negative weights (smaller $d$), resulting in a greater variation in the hardness of the instances induced by the stochasticity of both the generation method and the solver.

\begin{figure}
\begin{center}
\includegraphics[scale=0.6]{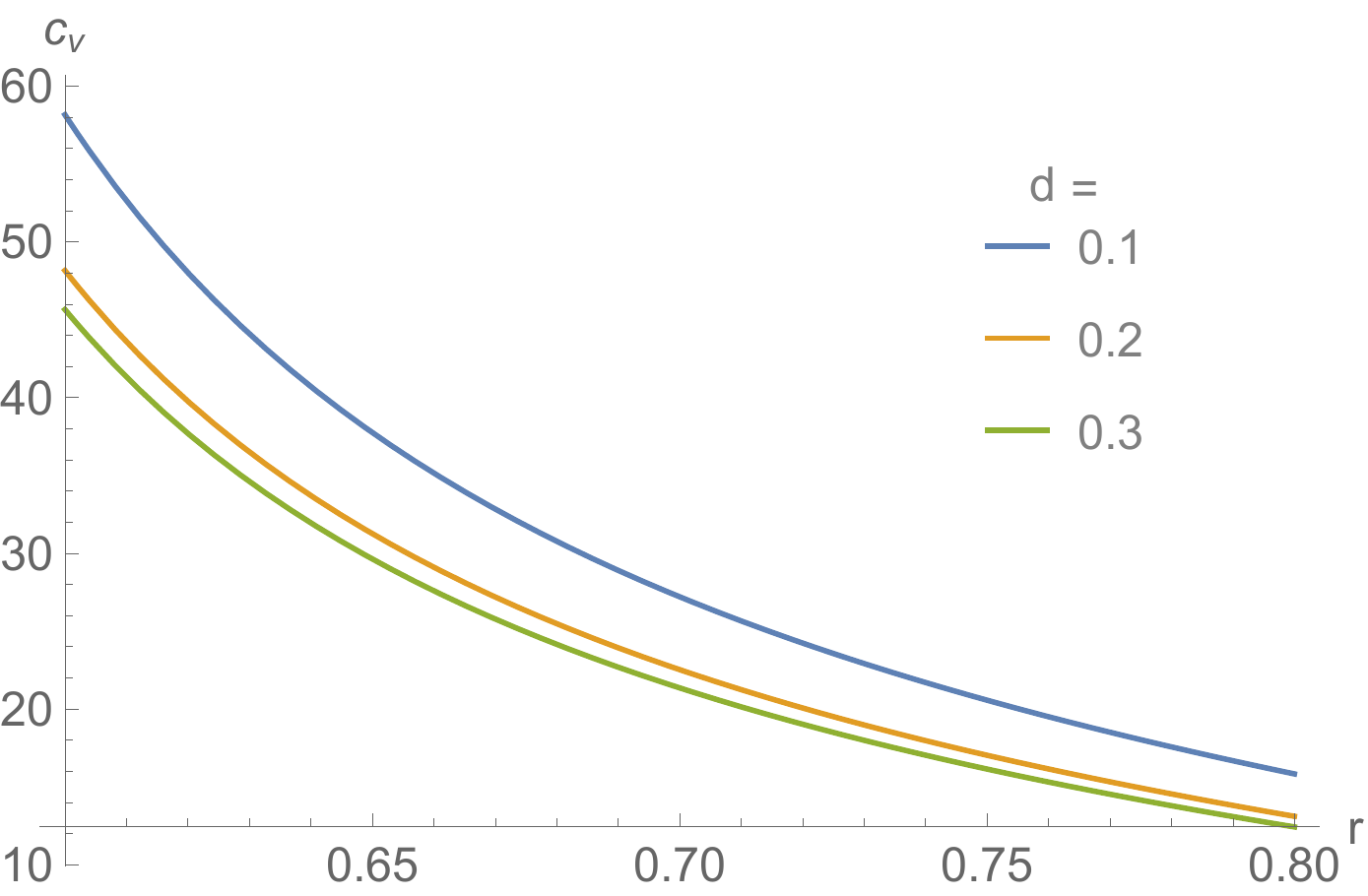}
\caption{\label{disperse} The dispersion of the local field ($c_v$) with respect to the alignment of the spins ($r$) for different concentrations of negative weights ($d$). The parameters, $\{n,N,\epsilon\}=\{1000,1000,0.01\}$ are chosen to produce the plots, noting that the loop density is extensive.}
\end{center}
\end{figure}

\section{Simulated Annealing}
\label{appendix:simu}

We first assign the following probability to each spin state $\mbf{s}$
\begin{equation*}
p(\beta,\mbf{s}) = e^{-\beta E(\mbf{s})},
\end{equation*}
where $\beta$ is interpreted as the inverse temperature of the system. The simulated annealing (SA) algorithm can be thought of as a Metropolis-Hastings sampling algorithm on a PMF varying in time (based on the $\beta$ schedule). Recall that the acceptance ratio for the Metropolis-Hastings algorithm is
\begin{equation*}
A(\mbf{s},\mbf{s'}) = \min\Big( 1,\frac{p(\mbf{s'})}{p(\mbf{s})} \Big) = \min\Big( 1,e^{-\beta\big( E(s') - E(s) \big)} \Big).
\end{equation*}
An iteration of the SA algorithm over the spins is referred to as a {\it sweep}. A sweep consists of performing single-spin flips over all the spins in the visible layer, followed by spin flips over the hidden layer. Usually, solving a non-trivial RBM instance using SA requires multiple sweeps. Since this is a single-spin flip algorithm, we focus on the energy difference of a single spin flip. Recall that the RBM energy is given by
\begin{equation*}
E(\mbf{s}) = -\big( \sum_ {ij} W_{ij}v_ih_j + \sum_i a_iv_i + \sum_j b_jh_j \big),
\end{equation*}
so the energy change from flipping the spin $v_i$ is given by
\begin{equation*}
\begin{split}
&-E(v_i',\mbf{h})+E(v_i,\mbf{h}) \\
= \,& -a_i(v_i'-v_i)+\sum_j W_{ij}(v_i'-v_i)h_j \\
= \,& -2v_i'(a_i + \sum_j W_{ij}h_j) \\
= \,& -2v_i'\theta_i(\mbf{h}),
\end{split}
\end{equation*}
where $v_i' = -v_i$, and we denoted $\boldsymbol{\theta}(\mbf{h}) = \mathbf{W}\mathbf{h}$. Similarly, the energy change from flipping the spin $h_j$ is given by
\begin{equation*}
-E(v,h_j') + E(v,h_j) = -2h_j'\phi_j(\mbf{v}),
\end{equation*}
where we define $\boldsymbol{\phi}(\mbf{v}) = \mathbf{W}^T\mathbf{v}$. Then, the acceptance ratio can be written as
\begin{equation*}
\begin{split}
A(v_i,v_i') &= \min(1,e^{2\beta v_i'\theta_i}), \\
A(h_j,h_j') &= \min(1,e^{2\beta h_j'\phi_j}).
\end{split}
\end{equation*}

If the total number of sweeps for a given run is $N_{sweep}$, we can then set $\beta \in [\beta_{min}, \beta_{max}]$ to follow a linearly increasing schedule, or
\begin{equation*}
\beta = \beta_{min} + \frac{c-1}{N_{sweep}-1}(\beta_{max} - \beta_{min}),
\end{equation*}
We provide a basic pseudo-code for the SA algorithm in Algorithm \ref{sian}, where the angle updates for a single spin flip are given by
\begin{equation*}
\begin{split}
\Delta \theta_i = W_{ij}(h_j'-h_j) = 2 W_{ij} h_j', \\
\Delta \phi_j = W_{ij}(v_i'-v_i) = 2 W_{ij} v_i'.
\end{split}
\end{equation*}

\begin{algorithm}[H]
\caption{Simulated Annealing on RBMs}\label{sian}
\begin{algorithmic}[1]

\State $\text{Initialize a random spin configuration } \mathbf{v}^{(0)}, \mathbf{h}^{(0)}$
\State $\boldsymbol{\theta^{(0)}} = \mathbf{W}\mathbf{h^{(0)}}, \boldsymbol{\phi^{(0)}} = \mathbf{W^T}\mathbf{v^{(0)}}, E^{(0)}=E(\mathbf{v^{(0)}},\mathbf{h^{(0)}})$
\For{$c \in [1,N_{sweep}]$} 
\State $\beta = \beta_{min} + (\beta_{max} - \beta_{min})\frac{c-1}{N_{sweep}-1}$
\For{$i\in [1,n]$}
\State $v_i = -v_i, A = \min(1,e^{2\beta v_i\theta_i^{(c-1)}})$
\If {$rand()<A$}
\State $\text{Get new } \boldsymbol{\phi^{(c)}}, E^{(c)}$
\Else
\State $v_i = -v_i$
\EndIf
\EndFor
\For{$j\in [1,m]$}
\State $h_j = -h_j, A = \min(1,e^{2\beta h_j\phi_j^{(c-1)}})$
\If {$rand()<A$}
\State $\text{Get new } \boldsymbol{\theta^{(c)}}, E^{(c)}$
\Else
\State $h_j = -h_j$
\EndIf
\EndFor
\EndFor

\end{algorithmic}
\end{algorithm}

Assuming that $m$ scales linearly with $n$, then the size of the RBM (total number of spins) is of the order $O(n)$. Flipping a single spin requires updating the entire $\boldsymbol{\theta}$ or $\boldsymbol{\phi}$ vector, so the time complexity of performing a single spin update is $O(n)$. Performing a sweep requires flipping all the spins of the RBM, so the time complexity of a single sweep is $O(n^2)$.

\section{Optimal Sweep Schedule}
\label{appendix:sweep}

For a given triplet of the parameters $\{n,f,\rho\}$, it is possible to generate multiple random RBM instances; for each instance, we record the total number of sweeps required to find the solution as $N_{tot}$. This gives us multiple values of $N_{tot}$ each corresponding to an RBM instance generated. It can be checked that the distribution of $N_{tot}$ follows approximately a log-normal distribution. If we have $k$ samples of $N_{tot}$, the estimator of the log mean of $N_{tot}$ is given as
\begin{equation*}
\hat{\mu} = \frac{1}{k}\sum_{i=1}^k \log\big( N_{tot,i} \big),
\end{equation*}
and the estimator of the log standard deviation of $N_{tot}$ is given as
\begin{equation*}
\hat{\sigma} = \sqrt{ \frac{1}{k-1}\sum_{i=1}^k \Big( \log\big( N_{tot,i} \big) -\hat{\mu} \Big)^2 }.
\end{equation*}
We can then estimate the 5th and 95th percentile of the distribution of $N_{tot}$ respectively as
\begin{equation*}
N_{tot,5\%} = \exp(\hat{\mu} - 2\hat{\sigma})
\qquad
N_{tot,95\%} = \exp( \hat{\mu} + 2\hat{\sigma} ).
\end{equation*}
For the remainder of this section, the 95th percentile of $N_{tot}$ is assumed whenever we refer to $N_{tot}$.

To ensure that the hardness is accurately measured for the RBM instance, we have to ensure that we are using the optimal SA parameters for every given triplet $\{n,f,\rho\}$, which is equivalent to finding the SA parameters that result in the smallest $N_{tot}$ value. This generally incurs a large amount of computational expense as it involves multiple runs under different SA parameters. So for this work, we focus on finding the optimal SA parameters for easy instances (with small $\{n,f\}$) such that the scaling behavior can be determined and the optimal SA parameters for harder instances can be extrapolated. 

Our focus here is to determine the scaling behavior of the optimal $N_{sweep}$ with respect to $\{n,f\}$ at the hardest loop density. This is performed in two iterations, a crude optimization followed by a finer one. In the first iteration, we essentially assume that the hardness peak is at $\rho_0 = 0.47$, and proceed to determine the scaling behavior of $N_{sweep}$ with respect to $\{n,f\}$. Then, using this non-optimal $N_{sweep}$, we can attempt to determine the relationship between the location of the actual hardness peak, $\rho_0$, and the system size, $n$. Note that it is unimportant that $N_{sweep}$ is not yet fully optimized at this stage as the location of the hardness peak is rather insensitive to the choice of $N_{sweep}$ (see Section \ref{Section_density}). After determining the hardness peaks for various system sizes, we then proceed with the second iteration of the optimization, where a more accurate scaling behavior of $N_{sweep}$ with respect to $\{n,f\}$ is determined by setting $\rho$ properly at the hardness peak for every system size. In theory, this procedure can be re-iterated as many times as needed. However, in practice, we find that the coefficient estimators of the scaling law already converge within their mean-square errors only after two iterations.

We now present the results for the two iterations. In the first iteration, we fix the loop density at $\rho = 0.47$ and let $n\in \{30,40,50,60,70,80\}$ and $f\in \{0.05,0.075,0.1,0.125,0.15,0.175\}$ so the instances are sufficiently easy to solve within a reasonable amount of time. For each pair of $\{n,f\}$, we generate 10000 different RBM instances, and try to find the optimal $N_{sweep}$ that minimizes the 95th percentile of $N_{tot}$. We fit the relationship between the optimal $N_{sweep}$ and $\{n,f\}$ with a product of two polynomials corresponding to the two parameters:
\begin{equation*}
\begin{split}
N_{sweep}(n,f) &= (0.504n^2-13.3n+311) \\
&\times (193f^3-52.7f^2+4.73f-0.102).
\end{split}
\end{equation*} 
We then use this value of $N_{sweep}$ to determine the scaling behavior of the hardest density with respect to the system size (see Section \ref{Section_density}). In the second iteration, we then use the optimal value of $N_{sweep}$ properly optimized at the hardness peak to derive more accurately the following fitting function
\begin{equation*}
\begin{split}
N_{sweep} &= (1.29n^2-33.1n+1664) \\
&\times (41.4f^3-11.7f^2+1.06f-0.018).
\end{split}
\end{equation*} 
This sweep schedule is used to study the phase transition induced by the frustration index presented in Section \ref{Section_frus}.

%\section{Second Hardness Peak}
%\label{appendix:2hard}
%
%In addition to the hardness peak at the expected density, we observe a second more dominant hardness peak at a higher density for the structured loop algorithm. We here provide a heuristic discussion of why the second hardness peak appears. As discussed in Appendix \ref{appendix:l_max}, for instances generated at $d>0.5$, the planted local minimum is relatively attractive in its neighborhood cluster. This essentially means that a local solver in the metastable cluster will fall into the local minimum with high probability. However, this is only relevant if the solver reaches the metastable cluster in the first place, and the likelihood of this occurring will depend strongly on the loop density.
%
%Following the notation in Appendix \ref{appendix:l_max}, we let $B_1 = I \times J$, and the planted local minimum is related to the ground state by $B_1 \cup B_4$. If the loop density is extensive, or of order $O(n)$, then a random initialization of spins will already likely produce a local field pointing in the ground state for every spin. This is simply a consequence of the LLN as the gauged weight matrix is dominated by positive elements. In other words, a randomly initialized local solver will follow the general 

\raggedbottom
\pagebreak

\bibliographystyle{apsrev4-1}
\bibliography{SUSY_ref}

\end{document}